\documentclass[journal]{IEEEtran}
\usepackage{amsmath,amsfonts}
\usepackage{algorithmic}
\usepackage{algorithm}
\usepackage{array}
\usepackage[caption=false,font=normalsize,labelfont=sf,textfont=sf]{subfig}
\usepackage{textcomp}
\usepackage{stfloats}
\usepackage{url}
\usepackage{verbatim}
\usepackage{graphicx}
\usepackage{natbib}
\setcitestyle{numbers,square}
\usepackage{makecell}
\usepackage{booktabs}
\usepackage{multirow}
\usepackage{amsthm}
\usepackage[misc]{ifsym}
\usepackage[colorlinks,linkcolor=blue]{hyperref}
\usepackage{algorithm}
\usepackage{algorithmic}
 
\makeatletter
\newcommand{\removelatexerror}{\let\@latex@error\@gobble}
\makeatother

\usepackage{multirow}
\usepackage{stfloats}
\usepackage{colortbl}
\usepackage{makecell}
\usepackage{thmtools}
\usepackage{thm-restate}
\usepackage[table]{xcolor} 

\newcommand{\etal}{\textit{et al}.}
\newcommand{\ie}{\textit{i}.\textit{e}.}

\newtheorem{proposition}{Proposition}

\begin{document}

\title{Semantic-Aligned  Adversarial Evolution Triangle for High-Transferability Vision-Language Attack}

\author{Xiaojun Jia~\textsuperscript{$\star$},
        Sensen Gao~\textsuperscript{$\star$},
        Qing Guo~\textsuperscript{\Letter},
        Ke Ma,
        Yihao Huang, 
        Simeng Qin,  \\
        Yang Liu,~\IEEEmembership{Senior Member,~IEEE},
        Ivor Tsang~\IEEEmembership{Fellow,~IEEE},
        and~Xiaochun Cao~\textsuperscript{\Letter}~\IEEEmembership{Senior Member,~IEEE}
\IEEEcompsocitemizethanks{
\IEEEcompsocthanksitem  Xiaojun Jia , Yihao Huang, Yang Liu are with Nanyang Technological University, Singapore (e-mail: jiaxiaojunqaq@gmail.com, huangyihao22@gmail.com, yangliu@ntu.edu.sg)
\IEEEcompsocthanksitem Sensen Gao is with the Department of Computer Vision, Mohamed Bin Zayed University of Artificial Intelligence. (e-mail: Sensen.Gao@mbzuai.ac.ae)
 \IEEEcompsocthanksitem Qing Guo and Ivor Tsang are with IHPC and CFAR, Agency for Science, Technology and Research, Singapore. (e-mail: tsingqguo@ieee.org and ivor\_tsang@cfar.a-star.edu.sg)
 \IEEEcompsocthanksitem Ke Ma is with the School of Electronic, Electrical and Communication Engineering, University of Chinese Academy of Sciences, Beijing 100049, China. (e-mail:  make@ucas.ac.cn) 
\IEEEcompsocthanksitem Simeng Qin is with Northeastern University, Shenyang, Liaoning, China (e-mail: qinsimeng670@gmail.com)
\IEEEcompsocthanksitem Xiaochun Cao is with the School of Cyber Science and Technology, Shenzhen Campus, Sun Yat-sen University, Shenzhen 518107, China. (e-mail: caoxiaochun@mail.sysu.edu.cn)
\IEEEcompsocthanksitem \textsuperscript{$\star$} Xiaojun Jia and Sensen Gao contribute equally to this work. \textsuperscript{\Letter} Qing Guo and Xiaochun Cao are corresponding authors

}
}
\markboth{Manuscript for IEEE Transactions on Pattern Analysis and Machine Intelligence}%
{Shell \MakeLowercase{\textit{et al.}}: A Sample Article Using IEEEtran.cls for IEEE Journals}


\maketitle

\begin{abstract}
 Vision-language pre-training (VLP) models excel at interpreting both images and text but remain vulnerable to multimodal adversarial examples (AEs). Advancing the generation of transferable AEs, which succeed across unseen models, is key to developing more robust and practical VLP models. Previous approaches augment image-text pairs to enhance diversity within the adversarial example generation process, aiming to improve transferability by expanding the contrast space of image-text features. However, these methods focus solely on diversity around the current AEs, yielding limited gains in transferability. To address this issue, we propose to increase the diversity of AEs by leveraging the intersection regions along the adversarial trajectory during optimization. Specifically, we propose sampling from adversarial evolution triangles composed of clean, historical, and current adversarial examples to enhance adversarial diversity. We provide a theoretical analysis to demonstrate the effectiveness of the proposed adversarial evolution triangle. Moreover, we find that redundant inactive dimensions can dominate similarity calculations, distorting feature matching and making AEs model-dependent with reduced transferability. Hence, we propose to generate AEs in the semantic image-text feature contrast space, which can project the original feature space into a semantic corpus subspace.  The proposed semantic-aligned subspace can reduce the image feature redundancy, thereby improving adversarial transferability.  Extensive experiments across different datasets and models demonstrate that the proposed method can effectively improve adversarial transferability and outperform state-of-the-art adversarial attack methods. The code is released at \href{https://github.com/jiaxiaojunQAQ/SA-AET}{https://github.com/jiaxiaojunQAQ/SA-AET}.
\end{abstract}


\section{Introduction}
\IEEEPARstart{V}{ision}-language pre-training (VLP) models achieve excellent performance across various downstream Vision-and-Language tasks, such as visual entailment~\cite{shi2021dense}, visual grounding~\cite{lei2021understanding}, image captioning~\cite{Hu_2022_CVPR}, and image-text retrieval~\cite{khan2021exploiting}. However, they have been found to be vulnerable to adversarial examples~\cite{zhang2022towards,lu2023setlevel,han2023ot,he2023sa,cheng2024typography,luo2024image,gao2024adversarial}.  Exploring adversarial vulnerabilities can inspire additional research dedicated to developing more robust and applicable VLP models. 

\begin{figure}[t]
    \centering
    \vspace{-5mm}
    \includegraphics[width=1.2\linewidth,bb=0 0 960 540]{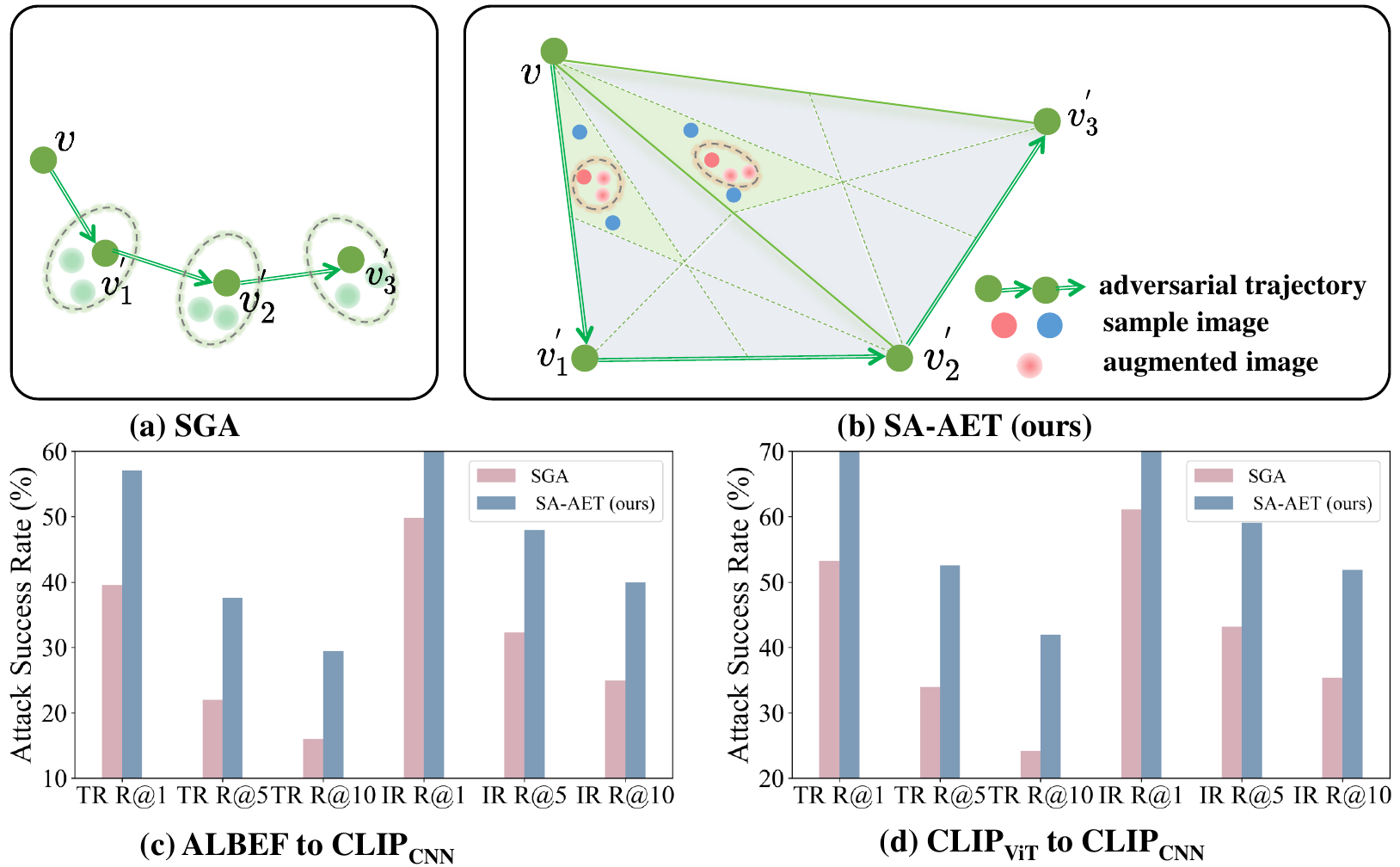}
    \caption{\textbf{Comparison of Our Method and Set-Level Guided Attack (SGA)~\cite{lu2023setlevel}.} (a) illustrates the main concept of SGA, which involves performing data augmentations around online adversarial examples. (b) demonstrates the core idea of our SA-AET, where data augmentations are applied within the adversarial sub-triangle. The red and blue dots represent images sampled from this sub-triangle, with red dots highlighting the optimal samples chosen through a text-guided adversarial example selection strategy. The surrounding light red dots represent resized augmentations applied to these optimal samples, similar to the strategy used in SGA. (c) and (d) compare the adversarial transferability of our SA-AET against SGA using adversarial examples from ALBEF \cite{li2021align} and CLIP$_\text{ViT}$~\cite{radford2021learning} to attack CLIP$_\text{CNN}$~\cite{radford2021learning}, respectively.  }
    \label{fig:home}
\end{figure}
\begin{figure*}[t]
    \centering
    \includegraphics[width=1.0\linewidth]{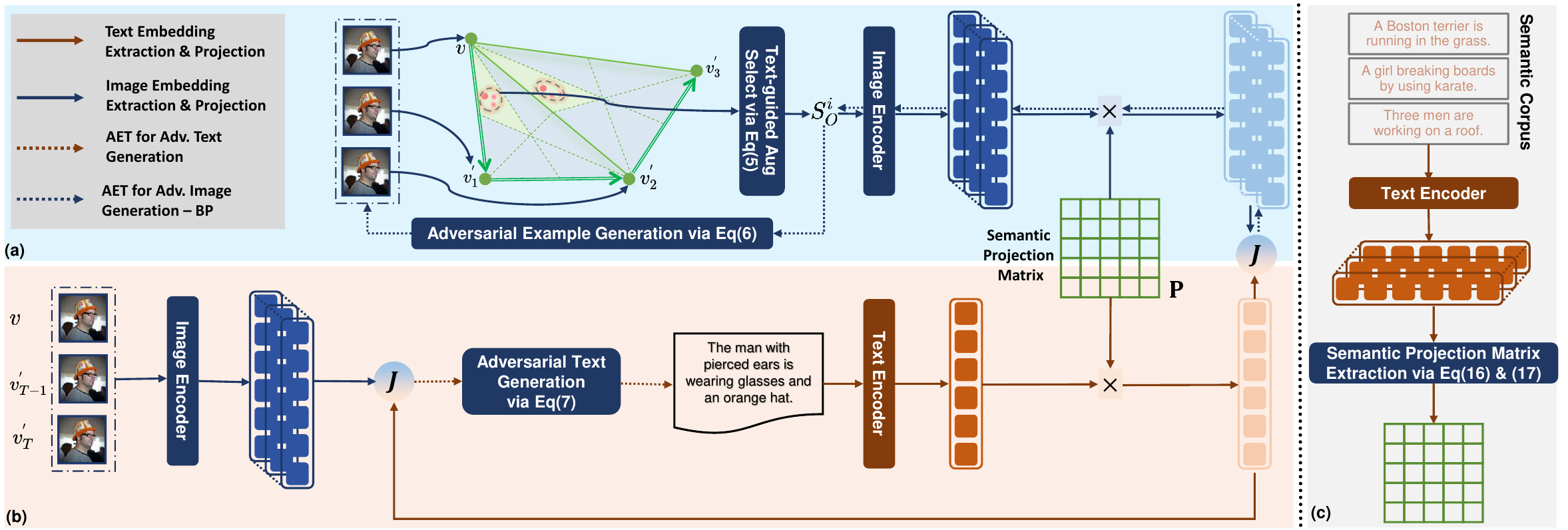}
    \caption{The Pipeline of the Proposed SA-AET: (a) Pipeline for the Adversarial Evolution Triangle (AET) in Adversarial Image Generation. (b) Pipeline for the Adversarial Evolution Triangle (AET) in Adversarial Text Generation. (c) Pipeline for Extracting the Semantic Projection Matrix. }
    \label{fig:framework}
\end{figure*}
\par Previous works mainly concentrate on exploring generating adversarial examples for VLP models in a white-box setting, in which the attacker can access model internal information, such as model parameters, etc. Some studies~\cite{gu2023survey} have shown that adversarial examples generated on a victim model can successfully attack unseen target models, a phenomenon known as adversarial transferability. Given the limited access to detailed model structures in real-world scenarios, it is crucial to investigate the transferability of multimodal adversarial examples~\cite{zhang2024does,zhu2024learning}. A series of works focus on generating adversarial examples with high transferability for VLP models. For example, Lu \textit{et al.}~\cite{lu2023setlevel} propose to improve adversarial transferability for VLP models by introducing input diversity through data augmentation. Although previous works have achieved some effectiveness in boosting adversarial transferability in vision-language attacks, they mainly focus on maximizing the contrastive loss function in the image-text feature space to generate adversarial examples and increasing their diversity along the optimization path to improve adversarial transferability. While these methods predominantly enhance diversity in online adversarial examples, they still have a lot of room for improvement in improving adversarial transferability. 
\par Specifically, as shown in Figure~\ref{fig:home}, during each optimization iteration of the attack, previous works~\cite{lu2023setlevel,he2023sa} perform data augmentation on the generated adversarial examples (\textit{i.e.,} online adversarial image) to improve transferability. This adversarial strategy increases the variety of adversarial examples throughout the optimization pathway, consequently yielding improvements in their transferability. Nonetheless, the strategy still risks overfitting the victim model due to the heavy reliance on examples from the adversarial trajectory, which results in reduced attack success rates when the adversarial examples are transferred to other VLP models.

\par To address these issues, as shown in Figure~\ref{fig:framework} (a) and (b), we first propose adopting the intersection evolution triangles along the adversarial trajectory to enhance the diversity of adversarial examples during optimization. Specifically, in each attack iteration, we propose constructing an adversarial evolution triangle that incorporates the original image, the adversarial image from the previous step, and the current adversarial image. Then, we propose to circumvent overfitting by strategic sampling within this evolution triangle, thereby avoiding excessive focus on adversarial example diversity only around adversarial images. We explore the impact of different adversarial evolution sub-triangle sampling on adversarial transferability and propose to sample from the adversarial evolution sub-triangle close to clean examples and previous adversarial examples. After obtaining the sampling samples, we generate the adversarial perturbations on the samples to stay away from the text. Subsequently, we apply these perturbations to current adversarial examples and select the one that diverges most significantly from the text. We also propose generating adversarial text that deviates from the previous adversarial evolution triangle along the optimization path rather than being distant from the last adversarial image example. We provide a theoretical analysis to illustrate the effectiveness of the proposed adversarial evolution triangle.

\par In addition, image feature embeddings usually contain much information irrelevant to text features~\cite{liang2022mind,zhou2023clip,hu2024reclip}. In the original image-text contrast space, only a limited number of feature dimensions in the image features may be activated, while the remaining dimensions remain inactive and redundant. These redundant dimensions may dominate the similarity calculation, potentially distorting the feature matching of images and texts.
This feature distortion makes the generated adversarial examples highly dependent on the victim model, reducing their transferability. Hence, generating adversarial examples within the native image-text feature contrast space increases the likelihood of overfitting the victim model, diminishing the examples' transferability. 
 To overcome this limitation, we propose generating adversarial examples within the semantic image-text feature contrast space, which maps the original feature space into a semantic corpus subspace. Specifically, as shown in Figure~\ref{fig:framework} (c), we construct a semantic subspace through a series of independent text descriptions and project the image features of the original space into the semantic subspace. Then, we maximize the contrast loss between images and text in the semantic subspace to generate adversarial perturbations.

\par By assembling the proposed methods, we conduct our vision-language attack by exploiting semantically aligned adversarial evolution triangle, \textit{i.e.,} SA-AET. We conduct a series of experiments to evaluate the effectiveness of the proposed method on the two widely used multimodal datasets, consisting of  Flickr30K~\cite{plummer2015flickr30k} and MSCOCO~\cite{lin2014microsoft}. The evaluation experiments are also conducted on three
vision-and-language downstream tasks, which include image-text retrieval (ITR), visual grounding (VG), and image captioning (IC). The experimental results indicate that the proposed method can significantly improve the transferability of multimodal adversarial examples, surpassing the state-of-the-art adversarial attack methods. Furthermore, when adversarial examples generated from ITR by the proposed method are applied to other vision-and-language downstream tasks, attack performance is significantly enhanced. Our main contributions are in five aspects:
\begin{itemize}
    \item We propose enhancing the diversity of adversarial examples during optimization by leveraging the intersection evolution triangle of adversarial trajectories, thereby improving the transferability of multimodal adversarial examples against VLP models. Furthermore, we provide a theoretical analysis to support the proposed adversarial evolution triangle. 
    \item We investigate how sampling from different adversarial evolution sub-triangles affects adversarial transferability and propose sampling from the evolution sub-triangle that is close to clean examples and previous adversarial examples. 
    \item We propose to generate the adversarial text by deviating from the final adversarial evolution triangle along the optimization trajectory rather than the final adversarial example, minimizing overfitting on the surrogate model to improve its transferability.
     \item To further enhance the transferability of adversarial examples, we propose generating them in the semantic image-text feature contrast space by mapping the original feature space onto a subspace defined by a semantic corpus.
    \item Our extensive experiments across various network architectures and datasets demonstrate that the proposed method significantly enhances the transferability of multimodal adversarial examples and outperforms state-of-the-art
    multimodal transfer adversarial attack methods. 
    
\end{itemize}

\par This paper is a journal extension of our conference paper~\cite{gao2024boosting} (called DRA). Compared to the preliminary conference version, we have made significant improvements and extensions in this version. The main differences are in four aspects: 1) In addition to sampling the intersection evolution triangle of adversarial trajectories proposed in the previous version, we explore the impact of different sampling strategies and propose to sample adversarial evolution triangles close to clean examples and previously generated adversarial examples in \textbf{Section~\ref{sec:sub_adversarial_region}}. This can further improve the transferability of multimodal adversarial examples. We add a theoretical analysis to demonstrate the effectiveness of the proposed adversarial evolution triangle in \textbf{Section~\ref{sec:theoretical}}. 2) We propose to generate the adversarial examples in the semantic image-text feature contrast space in \textbf{Section~\ref{sec:proposed_space}}, which can reduce reliance on victim models, thereby improving transferability. 3) We conduct more experiments and analyses, which include comparisons with state-of-the-art methods, ablation studies, and performance analyses. We adopt some state-of-the-art adversarial attack methods as the new comparison in \textbf{Section~\ref{sec:comparison}}. We add the ablation study versus the different proposed elements in \textbf{Section~\ref{sec:ablation}}. We analyze the effective performance of the proposed method in \textbf{Section~\ref{sec:perfromance_analysis}}. 4) We have thoroughly revised the abstract, introduction, method, experiment, and conclusion sections to offer a more detailed overview of our motivation and approach. Furthermore, we have updated all figures and tables to enhance clarity and presentation.
\section{Related Work}
In this section, we begin by discussing the existing research on vision-language pre-training models. Subsequently, we explore the studies related to downstream vision and language tasks. Finally, we introduce the research concerning the transferability of multimodal adversarial examples for vision-language pre-training models.

\subsection{Vision-Language Pre-training Models}
Vision-language pre-training (VLP) models enhance various Vision-and-Language (V+L) tasks by using multimodal learning from extensive image-text pairs \cite{li2022blip}. Initially, VLP models heavily relied on pre-trained object detectors to generate multimodal representations \cite{chen2020uniter,li2020oscar,zhang2021vinvl,wang2022vlmixer,tan2019lxmert}. However, recent developments have seen a shift towards employing end-to-end image encoders such as the Vision Transformer (ViT) \cite{dosovitskiy2010image,touvron2021training,yuan2021tokens}, which provide faster inference speeds. Consequently, some recent studies suggest replacing the computationally intensive object detectors with these more efficient image encoders \cite{dou2022empirical,li2022blip,li2021align,wang2023accelerating,yang2022vision}. The learning vision-language representations in VLP models can be divided into two categories: the fused architecture and the aligned architecture. Fused VLP models, such as ALBEF \cite{li2021align} and TCL \cite{yang2022vision}, first adopt two separate unimodal encoders to extract features from text and images. These features are combined using a multimodal encoder to create a joint representation. On the other hand, aligned VLP models, like CLIP \cite{radford2021learning}, aim to harmonize the feature spaces of different unimodal encoders, significantly enhancing performance in downstream tasks \cite{abdelfattah2023cdul}. In this paper, we focus on assessing our proposed method using various popular fused and aligned VLP models.

\subsection{Downstream Vision-and-Language Tasks}
The downstream vision-and-language tasks can be divided into three categories:
Image-text Retrieval (ITR), Visual Grounding (VG), and Image Captioning (IC). Image-text Retrieval (ITR): it involves retrieving relevant information, both textual and visual, in response to queries presented in a different modality \cite{chen2020imram,cheng2022vista,wang2019camp,zhang2020context}. Typically, it includes two sub-tasks: image-to-text retrieval and text-to-image retrieval. Specifically, for aligned VLP models, the Text Retrieval (TR) and Image Retrieval (IR) tasks utilize ranking based on the similarity between text and image embeddings. For fused VLP models, similarity scores are computed for all image-text pairs to identify the Top-N candidates since the embedding spaces of unimodal encoders are not aligned. The multimodal encoder then processes these candidates, calculating matching scores to determine the final ranking. Visual Grounding (VG): Identifying and mapping out specific areas within a visual scene corresponding to entities or concepts described in natural language. For example, in VLP models, ALBEF enhances Grad-CAM \cite{selvaraju2017grad} by leveraging the resulting attention map to prioritize the detected proposals. Image Captioning (IC): It involves creating a textual description that accurately reflects or conveys the content of a given visual input, usually through generating captions for images. To evaluate the performance of image captioning models, metrics such as BLEU \cite{naseer2021intriguing}, METEOR \cite{banerjee2005meteor}, ROUGE \cite{lin2004rouge}, CIDEr \cite{vedantam2015cider} and SPICE \cite{anderson2016spice} are commonly used. These metrics measure the quality and relevance of the generated captions by comparing them to reference captions.

\subsection{Adversarial Transferability}
Adversarial attack methods~\cite{jia2020adv,guo2020abba,guo2021learning,he2023generating,gu2023survey} can be divided into two categories: white-box attacks~\cite{jia2023improving,huang2023ala} and black-box attacks~\cite{bai2020improving,park2024hard}. The white-box attacks indicate that the attacker has full access to the model (\textit{e.g.,} model parameters and architecture), while the black-box attacks cannot. An increasing number of researchers focus on black-box attacks, as they are more practical for real-world applications where model information is often limited. In the context of image classification tasks, data augmentation techniques like TIM \cite{dong2019evading}, ADMIX \cite{wang2021admix}, and PAM \cite{zhang2023improving} are employed to enhance the transferability of adversarial examples. Recent works~\cite{lu2023setlevel,he2023sa, chung2024towards,fu2024improving} have begun to explore improving the transferability of multimodal adversarial examples for VLP models. One common method involves integrating unimodal adversarial attacks~\cite{madry2017towards,li2020bertattack} from each modality. For example, Sep-Attack~\cite{lu2023setlevel} directly combines BERT-Attack~\cite{li2020bertattack} for text and PGD~\cite{madry2017towards} for image attacks. Zhang \etal~\cite{zhang2022towards}
develop a white-box attack, called Co-Attack, to attack popular VLP models by considering cross-modal interactions. Building on this foundation, 
Lu \etal~\cite{lu2023setlevel} propose the SGA method to boost the transferability of multimodal adversarial examples by expanding a single image-text pair into diverse sets to increase adversarial example variety. However, SGA predominantly considers diversity near adversarial examples throughout the optimization, which may elevate the likelihood of overfitting to the targeted model and reduce the adversarial transferability. Consequently, this work aims to thoughtfully broaden the diversity of adversarial examples without excessively concentrating on diversity solely around adversarial images.
\section{The Proposed Method}
\begin{figure}[t]
    \centering
    \includegraphics[width=1.0\linewidth]{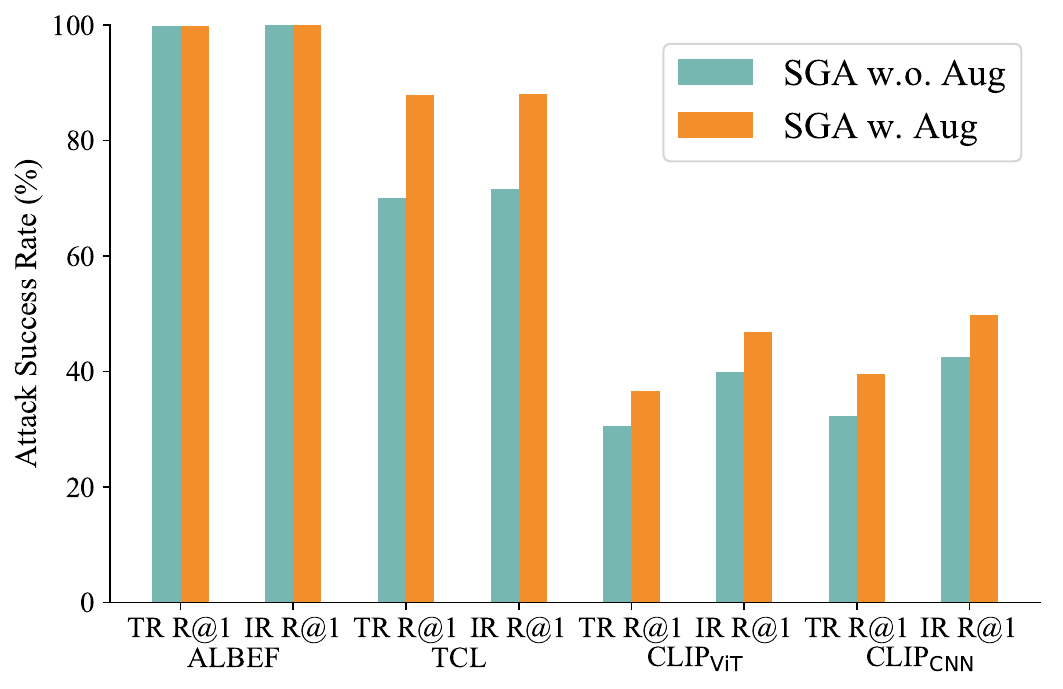}
    \caption{Attack Success Rate (\%) of SGA with and without Image Augmentation. The SGA w.o. Aug does not utilize image augmentation techniques. We employ ALBEF to generate multimodal adversarial examples.}
    \label{fig:aug_sga}
\end{figure}
\subsection{Motivation}
Adversarial attacks on VLP models typically cause a mismatch between adversarial images and their corresponding adversarial text while conforming to predefined limits on perturbations in both the image and text domains. Let $(\mathbf{x}_{I}, \mathbf{x}_{T}) \in \mathcal{D}$ denote an image-text pair extracted from a multimodal dataset $\mathcal{D}$. $(\tilde{\mathbf{x}}_{I}, \tilde{\mathbf{x}}_{T})$ represents the corresponding adversarial image-text pair in the image-text searching spaces $B\left[{\mathbf{x}}_{I}, \epsilon_I\right]$ and $B\left[{\mathbf{x}}_{T}, \epsilon_T\right]$, where $\epsilon_I$ represents the maximal perturbation bound for the image, and $\epsilon_T$ represents the maximal number of changeable words in the caption. $F_{I}(\cdot)$ and $F_{T}(\cdot)$ represent the image and text encoders of VLP models, respectively. The multimodal adversarial examples are generated by maximizing the loss function $J$ of VLP models. The objective function can be defined as:
\begin{equation}
\left\{\begin{array}{l}
\max J\left(F_I\left(\tilde{\mathbf{x}}_{I}\right), F_T\left(\tilde{\mathbf{x}}_{T}\right)\right) \\
\text {s.t.} \tilde{\mathbf{x}}_{I} \in B\left[{\mathbf{x}}_{I}, \epsilon_I\right], \tilde{\mathbf{x}}_{T} \in B\left[{\mathbf{x}}_{T}, \epsilon_T\right]. 
\end{array}\right.
\end{equation}
\par We adopt $J(\tilde{\mathbf{x}}_{I},\mathbf{x}_{T})$ to represent $J(F_I(\tilde{\mathbf{x}}_{I}), F_T(\mathbf{x}_{T}))$.
Previous works for improving the transferability of multimodal adversarial examples, specifically SGA~\cite{lu2023setlevel}, employ augmented image-text pairs to enhance the diversity of adversarial examples throughout the optimization process. During the generation of adversarial images, let $\tilde{\mathbf{x}}_{I}^{i}$ denote the adversarial image generated at the $i$-th step. 
In the next step $(i + 1)$, the SGA procedure begins by performing data augmentation (resize images to different scales) on $\tilde{\mathbf{x}}_{I}^{i}$, producing $n$ augmented examples $\left\{\tilde{\mathbf{x}}_{I}^{i 1}, \tilde{\mathbf{x}}_{I}^{i 2}, \ldots, \tilde{\mathbf{x}}_{I}^{i n}\right\}$. The whole process can be formulated as follows:
\begin{equation}
\begin{aligned}
\tilde{\mathbf{x}}_{I}^{i+1}=\underset{\mathbf{x}_{I},\epsilon_{I}}{\Pi}\left(\tilde{\mathbf{x}}_{I}^{i}+\alpha \cdot \operatorname{sign}\left(\frac{\nabla_{\mathbf{x}_{I}} \sum_{j=1}^n J(\tilde{\mathbf{x}}_{I}^{i j},\mathbf{x}_{T})}{\left\|\nabla_{\mathbf{x}_{I}} \sum_{j=1}^n J(\tilde{\mathbf{x}}_{I}^{i j},\mathbf{x}_{T})\right\|}\right)\right),
\end{aligned}
\end{equation}
where $\alpha$ represents the step size and $\underset{\mathbf{x}_{I},\epsilon_{I}}{\Pi}$ is to constrain adversarial examples $\mathbf{x}_{I}$ to stay within an $\epsilon_{I}$-radius of the clean image under the $L_{\infty}$ norm.
\par To better understand the role of image augmentation during the optimization phase in SGA, ALBEF is used as a surrogate model to generate multimodal adversarial examples. The generated adversarial examples are then applied to attack target VLP models, such as TCL and CLIP, to evaluate the transferability of the multimodal adversarial examples. The experimental results are shown in Figure~\ref{fig:aug_sga}. It can be observed that SGA can significantly boost the transferability of multimodal adversarial examples for VLP models, with improvements ranging from 6.14\% to 17.81\%. However, it is important to note that the attack success rates of the target models are still considerably lower than the source model. This discrepancy is mainly due to SGA's focus on promoting diversity around the adversarial example $\tilde{\mathbf{x}}_{I}^{i}$ during optimization while insufficiently accounting for the diversity of adversarial examples relative to the clean image. This can lead to overfitting on the victim model, thereby reducing adversarial transferability. 

\par To improve adversarial transferability, we propose to sample from adversarial evolution triangles composed of clean examples, historical adversarial examples, and current adversarial examples to boost the diversity of adversarial examples. Besides, we investigate the impact of various sampling strategies within adversarial evolution triangles on adversarial transferability and propose sampling from evolution triangles near clean examples and previously generated adversarial examples. Additionally, as for the text adversarial perturbations, our goal is to generate the adversarial perturbation by deviating from the adversarial image and adversarial evolution triangle, thereby reducing overfitting the source model and boosting the adversarial transferability. Moreover, we propose generating adversarial examples within the semantic image-text feature contrast space. This approach projects the original feature space into a semantically aligned subspace, reducing dependency on source models. The semantically coherent subspace enhances adversarial transferability across VLP models. 
\begin{figure}[t]
    \centering
    \includegraphics[width=1.0\linewidth]{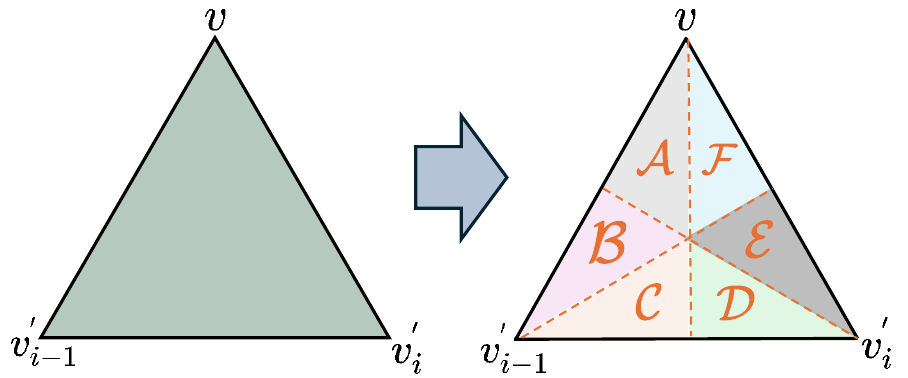}
    \caption{\textbf{Adversarial evolution sub-triangle partitioning.} $v$ represents the clean sample, $v_{i-1}^{'}$ represents the last adversarial example, and $v_{i}^{'}$ represents the current adversarial example. We conduct a more detailed investigation of this triangle by partitioning it into six sub-triangles based on the distance relationships among the clean example, the last adversarial example, and the current adversarial example.}
    \label{fig:sub_region}
\vspace{-5mm}
\end{figure}

\subsection{Adversarial Evolution Triangle for Adv Image Generation}
\label{sec:sub_adversarial_region}

\textbf{Adversarial evolution triangle (AET):} During the $i-$th iteration of the optimization, we can obtain the benign example $\mathbf{x}_{I}$, the adversarial example from the previous step $\tilde{\mathbf{x}}_{I}^{i-1}$ (previous adversarial example), and the current adversarial example $\tilde{\mathbf{x}}_{I}^{i}$. Then, these sample points together form a triangular region, referred to as the adversarial evolution triangle $\triangle \mathbf{x}_{I} \tilde{\mathbf{x}}_{I}^{i-1} \tilde{\mathbf{x}}_{I}^{i}$. 
Next, we can sample multiple instances from the adversarial evolution triangle $\triangle \mathbf{x}_{I} \tilde{\mathbf{x}}_{I}^{i-1} \tilde{\mathbf{x}}_{I}^{i}$, and denote the collection of samples as $\mathcal{S}^{i}=\left\{s_1^{i}, s_2^{i}, \ldots, s_m^{i}\right\}$. Each sample can be calculated as follows:
\begin{equation}\label{eq:region}
s_k=\lambda \cdot \mathbf{x}_{I}+\beta \cdot \tilde{\mathbf{x}}_{I}^{i-1}+\gamma \cdot \tilde{\mathbf{x}}_{I}^{i}, \text { where } \lambda+\beta+\gamma=1.0,
\end{equation}
where $\lambda$, $\beta$, and $\gamma$ represent the hyper-parameters, which decide how to do the sampling.

\par \textbf{Text-guided augmentation selection:} For our sampling set $\mathcal{S}^{i}$, an adversarial perturbation direction can be calculated for each image. The $k$-th sample is denoted as $s^{i}_{k}$. Its perturbation direction can be obtained through the following:
\begin{equation}
\epsilon^{i}_{k}=\alpha \cdot \operatorname{sign}\left(\frac{\nabla_{s^{i}_{k}} J\left(s^{i}_{k}, \mathbf{x}_{T}\right)}{\left\|\nabla_{s^{i}_{k}} J\left(s^{i}_{k}, \mathbf{x}_{T}\right)\right\|}\right).
\end{equation}
At this stage, we obtain a diverse set of perturbations, denoted as $\epsilon^{i} = \left\{\epsilon^{i}_{1}, \epsilon^{i}_{2}, \ldots, \epsilon^{i}_{m}\right\}$. To fully leverage modality interactions, we propose a text-guided augmentation selection approach to identify the optimal sample. Specifically, we integrate each element from the perturbation set $\epsilon^{i}$ into the adversarial image $\tilde{\mathbf{x}}_{I}^{i}$ individually. The objective is to determine which perturbation maximizes the distance between $\tilde{\mathbf{x}}_{I}^{i}$ and $\mathbf{x}_{T}$. This process can be formalized as follows:
\begin{equation}
o = \underset{\epsilon^{i}_{k}\in\epsilon^{i}}{\arg \max} J\left(\tilde{\mathbf{x}}_{I}^{i}+\epsilon^{i}_{k},\mathbf{x}_{T}\right).
\end{equation}
At this juncture, $s^{i}_{o}$ represents the optimal sample.

\par \textbf{Adversarial example generation:} 
Through sampling and text-guided filtering, we have obtained the optimal sample $s^{i}_{o}$ in the proposed adversarial evolution triangle $\triangle \mathbf{x}_{I} \tilde{\mathbf{x}}_{I}^{i-1} \tilde{\mathbf{x}}_{I}^{i}$. In addition, we follow SGA\cite{lu2023setlevel} as our baseline and integrate the image augmentation techniques (image resizing) during its optimization process. The chosen optimal sample $s^{i}_{o}$ is augmented and expanded into the set $\mathcal{S}^{i}_{O} = \left\{s^{i1}_{o},s^{i2}_{o}, \ldots, s^{in}_{o} \right\}$. Subsequently, we use the extended set $\mathcal{S}^{i}_{O}$ to generate adversarial perturbation and update $\tilde{\mathbf{x}}_{I}^{i}$ to $\tilde{\mathbf{x}}_{I}^{i+1}$. This process can be formalized as follows:
\begin{equation}
\tilde{\mathbf{x}}_{I}^{i+1}=\underset{\mathbf{x}_{I},\epsilon_{I}}{\Pi}\left(\tilde{\mathbf{x}}_{I}^{i}+\alpha \cdot \operatorname{sign}\left(\frac{\nabla_{s^{i}_{o}} \sum_{j=1}^n J\left(s^{ij}_{o}, \mathbf{x}_{T}\right)}{\left\|\nabla_{s^{i}_{o}} \sum_{j=1}^n J\left(s^{ij}_{o}, \mathbf{x}_{T}\right)\right\|}\right)\right).
\end{equation}


\par \textbf{Adversarial evolution sub-triangle:} To explore the impact of different adversarial evolution sub-triangle sampling on adversarial transferability, as shown in Figure~\ref{fig:sub_region}, we divide the entire adversarial evolution triangle into six evolution sub-triangles based on their proximity to three key points: the clean example, the previous adversarial example, and the current adversarial example. Each evolution sub-triangle is characterized by which of these points is the nearest and which is the second nearest. The evolution sub-triangles are defined as follows:
\begin{figure}[t]
    \centering
    \includegraphics[width=1.0\linewidth]{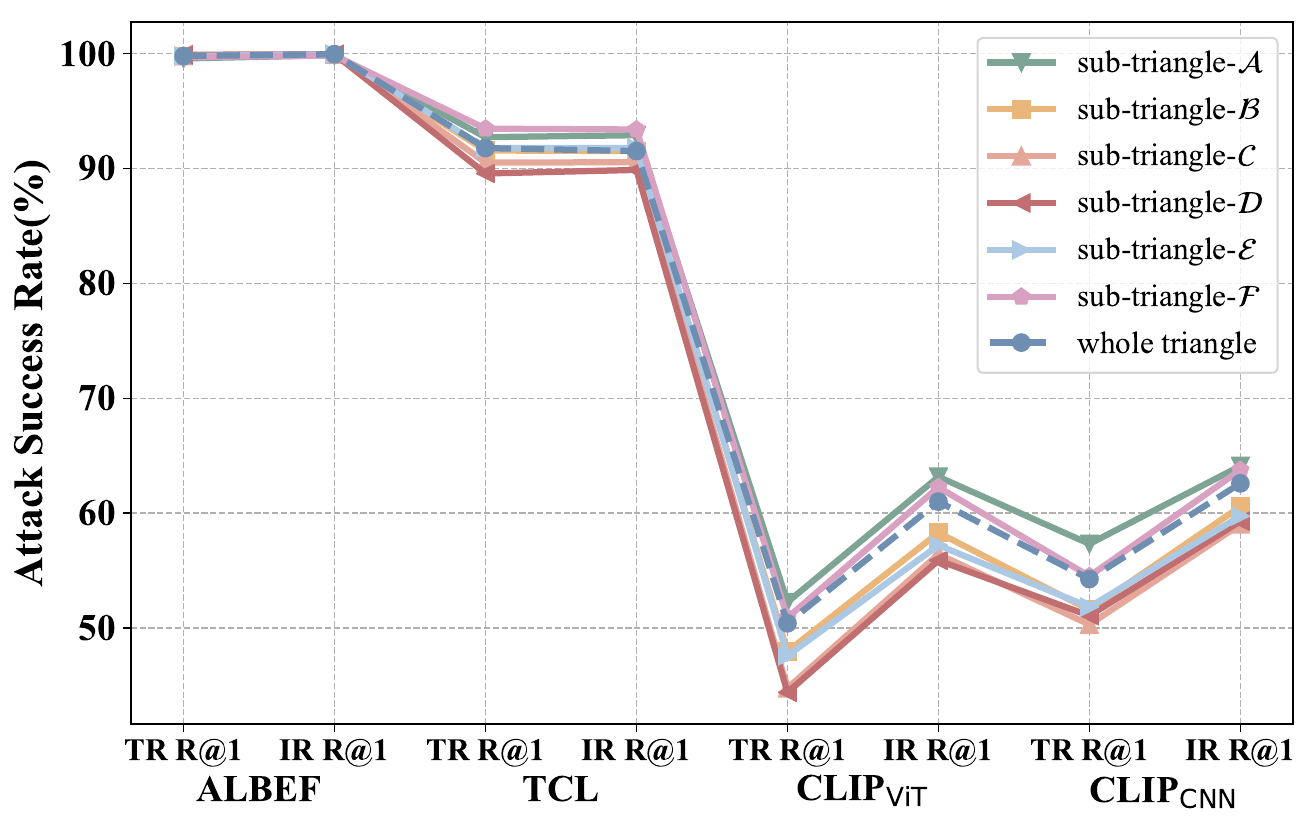}
    \caption{Attack Success Rate (\%) of different adversarial evolution sub-triangles, which is used to boost the diversity of adversarial examples for improving adversarial transferability.}
    \label{fig:sub_region_results}
\end{figure}
\begin{itemize}
  \item \textbf{Sub-triangle-$\mathcal{A}$:} 
  Nearest to the clean example; second nearest to the previous adversarial example.
  \item \textbf{Sub-triangle-$\mathcal{B}$:} 
  Nearest to the previous adversarial example; second nearest to the clean example.
  \item \textbf{Sub-triangle-$\mathcal{C}$:} Nearest to the previous adversarial example; second nearest to the current adversarial example.
  \item \textbf{Sub-triangle-$\mathcal{D}$:} Nearest to the current adversarial example; second nearest to the previous adversarial example.
  \item \textbf{Sub-triangle-$\mathcal{E}$:} Nearest to the current adversarial example; second nearest to the clean example.
  \item \textbf{Sub-triangle-$\mathcal{F}$:} Nearest to the clean example; second nearest to the current adversarial example.
\end{itemize}


We partition the adversarial evolution triangle into six distinct evolution sub-triangles by adjusting the hyperparameters $\beta$, $\gamma$, and $\lambda$. Specifically, these evolution sub-triangles are defined by all possible strict orderings of the three parameters:
sub-triangle-$\mathcal{A}$ is defined by $\gamma < \beta < \lambda$, sub-triangle-$\mathcal{B}$ by $\gamma < \lambda < \beta$, sub-triangle-$\mathcal{C}$ by $\lambda < \gamma < \beta$, sub-triangle-$\mathcal{D}$ by $\lambda < \beta < \gamma$, sub-triangle-$\mathcal{E}$ by $\beta < \lambda < \gamma$, and sub-triangle-$\mathcal{F}$ by $\beta < \gamma < \lambda$.
To study the impact of the adversarial evolution sub-triangles on adversarial transferability, we adopt ALBEF as the source model to generate multimodal adversarial examples and evaluate the adversarial transferability on other VLP models. The experiment results are shown in Figure~\ref{fig:sub_region_results}. Analyses are as follows. Different adversarial evolution sub-triangles can achieve different performances of adversarial transferability. Specifically, the transferability performance of sub-triangle-$\mathcal{C}$ is the lowest among all, while sub-triangle-$\mathcal{A}$ demonstrates higher transferability performance compared to both the whole triangle and the other sub-triangles. The results indicate that focusing sampling efforts on the sub-triangle near clean and previous adversarial examples significantly enhances adversarial transferability. 

\subsection{Adversarial Evolution Triangle for Adv Text Generation} 
As for the adversarial text generation, previous works generate adversarial texts in an iterative process by only deviating from the final generated adversarial image. Specifically, given the total iterations $T$ and the final adversarial image $\tilde{\mathbf{x}}_{I}^{T}$, the adversarial text $\tilde{\mathbf{x}}_{T}$ is generated by deviating from the features of the final adversarial image. Since the final adversarial example completely depends on the source model, the generated adversarial text is only far away from the final adversarial example, resulting in the final multimodal adversarial example overfitting the source model. To address this issue, we propose to deviate from the last adversarial evolution triangle for generating the adversarial text. The evolution triangle consists of the clean image ${\mathbf{x}}_{I}$, the previous adversarial example $\tilde{\mathbf{x}}_{I}^{T-1}$, and the current adversarial example $\tilde{\mathbf{x}}_{I}^{T}$. The adversarial text can be calculated as follows: 
\begin{equation}
\label{eq:text-deviate}
\begin{split}
\tilde{\mathbf{x}}_{T}  =  \underset{\tilde{\mathbf{x}}_{T} \in B\left[{\mathbf{x}}_{T}, \epsilon_t\right]}{\arg \max } & \left(\kappa \cdot J\left({\mathbf{x}}_{I}, \tilde{\mathbf{x}}_{T}\right) +\mu \cdot J\left(\tilde{\mathbf{x}}_{I}^{T-1},\tilde{\mathbf{x}}_{T}\right) \right.\\ &\left. +\nu \cdot J\left(\tilde{\mathbf{x}}_{I}^{T}, \tilde{\mathbf{x}}_{T}\right)\right),
\end{split}
\end{equation}
where $\kappa$, $\mu$, and $\nu$ represent the hyper-parameters used to generate the adversarial text. They are constrained such that $\kappa + \mu + \nu =1.0$. 
\begin{table*}[t]
\centering
\caption{Attack Success Rate (\%) of the proposed method with different proportions of text data in the dataset, which are used to generate the projection matrix. We use ALBEF to generate multimodal adversarial examples on the Flickr30K to evaluate adversarial transferability.}
\label{table:Proportion}
\setlength\tabcolsep{0.4cm}

\begin{tabular}{@{}ccccccccc@{}}
\toprule
\multirow{2}{*}{Proportion} & \multicolumn{2}{c}{ALBEF}       & \multicolumn{2}{c}{TCL}         & \multicolumn{2}{c}{CLIP$_\mathrm{ViT}$}   & \multicolumn{2}{c}{CLIP$_\mathrm{CNN}$}    \\ \cmidrule(l){2-9} 
                            & TR R@1         & IR R@1         & TR R@1         & IR R@1         & TR R@1        & IR R@1         & TR R@1         & IR R@1         
                            \\ \midrule
DRA~\cite{gao2024boosting}                    & 99.90           & 99.93          & 91.57          & 91.17          & 46.26         & 56.80           & 49.55          & 59.01          
\\ \midrule \midrule
Baseline & 99.90	& 99.93	& 91.57	& 91.17	& 46.26	& 56.80 & 49.55	& 59.01 \\
10\% & 100.0 & 100.0 & 94.73 & \textbf{94.69} & 47.73 & 58.96 & 53.38 & 61.51\\
20\% & 100.0 & 100.0 & 94.52 & 94.29 & 47.12 & 57.96 & 53.26 & 60.65\\
30\% & 100.0 & 100.0 & \textbf{95.15} & 94.26 & 47.98 & 58.02 & 53.51 & 60.79\\
\cellcolor{gray!40} 40\% & \cellcolor{gray!40}  \textbf{100.0} & \cellcolor{gray!40} 
 \textbf{100.0} & \cellcolor{gray!40}  94.84 & \cellcolor{gray!40} 
 94.10 & \cellcolor{gray!40} \textbf{49.08} & \cellcolor{gray!40} 
 \textbf{58.47} & \cellcolor{gray!40} \textbf{53.51} & \cellcolor{gray!40}  \textbf{61.51}\\
50\% & 100.0 & 100.0 & 94.52 & 94.19 & 48.34 & 57.51 & 51.47 & 61.20 \\
60\% & 100.0 & 100.0 & 95.15 & 94.05 & 47.24 & 58.25 & 52.11 & 60.82 \\
70\% & 100.0 & 100.0 & 94.73 & 94.55 & 47.85 & 57.64 & 52.49 & 60.62\\
80\% & 100.0 & 100.0 & 95.15 & 94.45 & 47.85 & 58.27 & 50.45 & 60.48\\
90\% & 100.0 & 100.0 & 94.42 & 94.26 & 47.24 & 57.64 & 52.11 & 60.48\\
100\% & 100.0 & 100.0 & 94.94 & 94.17 & 48.71 & 58.18 & 52.62 & 59.86\\
\bottomrule
\end{tabular}
\end{table*}

\subsection{Theoretical Analysis}
\label{sec:theoretical}
Here we provide theoretical analysis to show that the proposed method can improve the transferability of adversarial examples compared with the SGA~\cite{lu2023setlevel}. Wang \textit{et al.}~\cite{DBLP:conf/iclr/WangRLZ0Z21} introduce the Shapley value \cite{shapley:book1952} to analyze the interactions inside adversarial perturbations. They discover that the adversarial transferability and the interaction inside adversarial perturbation have a negative correlation, \textit{i.e.}, the adversarial examples with smaller interactions have better black-box transferability. These observations are also confirmed by the follow-up work like \cite{DBLP:journals/tip/WeiCGWGJD23,yu2023reliable,DBLP:journals/pami/YuanZJLS24}. We give the Shapley interaction indices of the proposed method and the SGA~\cite{lu2023setlevel} by the following theorem. 
\begin{restatable}[]{theorem}{goldbach}
    \label{thm:goldbach}
    The adversarial perturbations $\{\boldsymbol{\delta}_{t}\}$ generated by the proposed method are given as
    \begin{equation}
        \label{eq:delta_update}
        \boldsymbol{\delta}_{t} = \sum_{i=1}^{t} \boldsymbol{g}_i,\ \ t=2,\ 3,\ \dots\ ,
    \end{equation}
    where
    \begin{equation}
        \label{eq:g_update}
        \boldsymbol{g}_t = 
        \begin{cases}
            g(\boldsymbol{x}+\beta\cdot\boldsymbol{\delta}_{t-2}+\gamma\cdot \boldsymbol{\delta}_{t-1}),&\ \text{ if } t\geq 2,\\
            g(\boldsymbol{x}),&\ \text{ if }t=1,
        \end{cases}
    \end{equation}
    $g(\cdot)$ is the gradient of loss function $L(\cdot)$, $\beta,\ \gamma\in[0,1]$ are given constants, and $\boldsymbol{\delta}_{0} = (0,\dots,0)$. Meanwhile, the adversarial perturbations generated by the SGA~\cite{lu2023setlevel} $\{\boldsymbol{\zeta}_{t}\}$ \footnote{SGA~\cite{lu2023setlevel} can be treated as a special case of the proposed method as $\beta = 0$ and $\gamma = 1$.} are
    \begin{equation}
        \boldsymbol{\zeta}_{t} = \sum_{i=1}^{t} \boldsymbol{h}_i,\ \boldsymbol{h}_i = g(\boldsymbol{x}+\boldsymbol{\zeta}_{i-1}),\ t=1,\ 2,\ \dots\ .
    \end{equation}
    Then the interaction inside adversarial perturbation $\boldsymbol{\delta}_{t}$ will be
    \begin{equation}
        \label{eq:ours}
        \begin{aligned}
             & \mathbb{E}_{i,j}[\boldsymbol{I}_{i,j}(\boldsymbol{\delta}_{t})] &=&\ \ (\beta+\gamma)B\cdot t^3 + (A-2\beta B)\cdot t^2\\
             & & &\ \ + (2A-(\beta+\gamma)B)\cdot t\\
             & & &\ \ + A + 2\beta B,
        \end{aligned}
    \end{equation}
    where 
    \begin{equation}
        \begin{aligned}
            & A &=&\ \ \mathbb{E}_{i,j}[\boldsymbol{g}(i)\cdot\boldsymbol{g}(j)\cdot\boldsymbol{H}_{i,j}],\\
            & B &=&\ \ \boldsymbol{\mathbb{E}}_{i,j}[\boldsymbol{g}(i)\cdot\boldsymbol{H}_{i,j}\cdot\boldsymbol{g}^\top\boldsymbol{H}_{*j}],
        \end{aligned}
    \end{equation}
    $\boldsymbol{g}=[\boldsymbol{g}(1),\dots,\boldsymbol{g}(n)]$ is the gradient of $L(\boldsymbol{x})$, $\boldsymbol{H}_{i,j}$ is the $(i,j)$ element of the Hessian matrix $\boldsymbol{H}$ for $L(\boldsymbol{x})$, and $\boldsymbol{H}_{*j}$ is the $j$-th column of $\boldsymbol{H}$. Moreover, the interaction inside $\boldsymbol{\zeta}_{t}$
    \begin{equation}
        \label{eq:competitor}
        \mathbb{E}_{i,j}[\boldsymbol{I}_{i,j}(\boldsymbol{\zeta}_{t})] = B\cdot t^3 + (A-B)\cdot t^2.
    \end{equation}
\end{restatable}
Compared \eqref{eq:ours} with \eqref{eq:competitor}, it is clear that the interaction inside adversarial perturbation generated by the proposed method is lower than SGA~\cite{lu2023setlevel} as $\beta+\gamma<1$ and $t\geq 1$. Consequently, we know the proposed method will have better black-box transferability against the models with unknown parameters. We provide the proof details in the supplementary material.

\subsection{Contrast Space Optimization for Semantic-aligned AET}
\label{sec:proposed_space}
We generate multimodal adversarial examples by optimizing the maximization of the loss function $J$ of VLP models. Especially given the visual embeddings $F_I(\mathbf{x}_{I}) \in \mathbb{R}^d$,  and the text embeddings $F_T(\mathbf{x}_{T}) \in \mathbb{R}^d$, where $d$ represents the size of the feature embeddings, the loss function, abbreviated as $J$, in the original feature space can be calculated as:
\begin{equation}
J = \frac{F_I\left({\mathbf{x}}_{I}\right) \cdot F_T\left({\mathbf{x}}_{T}\right)}{d},
\end{equation}                                                             
where $\cdot$ represents the dot product of vectors, used to calculate the similarity between each pair of image and text features.

\par Previous works~\cite{zhou2023clip,hu2024reclip} have indicated that image feature embeddings often carry substantial information unrelated to text features. In the original image-text contrast space, only a small subset of feature dimensions within the image embeddings is typically active, while the rest remain inactive and redundant. These redundant dimensions can impact similarity calculations, potentially distorting the feature alignment between images and text. The feature distortion may cause the generated adversarial examples to rely highly on the specific victim model, reducing their transferability.
Hence, we propose a semantic
image-text feature contrast space to generate multi-mode adversarial examples. It maps the original
feature space into a semantic corpus subspace by using a semantic projection matrix. As shown in Figure~\ref{fig:framework} (c), assuming we have a $d$ dimensional semantic representation space $\mathbb{R}^{d}$, we have $\mathcal{N}$ semantic texts $U=\left[c_1, c_2, \ldots, c_\mathcal{N}\right]$. The text feature embeddings of those texts can be calculated as follows:
\begin{equation}
\mathcal{T}=\left[t_1, t_2, \ldots, t_\mathcal{N}\right] \in \mathcal{R}^{\mathcal{N} \times d}, 
\end{equation}
where $t_i= F_{T}(c_{i})$ for $i \in \{1,2, \ldots, \mathcal{N}\}$. Then, we perform singular value decomposition (svd)~\cite{klema1980singular} on these text features. It can be calculated as follows: 
\begin{equation}
\mathcal{U}, \mathcal{S}, \mathcal{V}=\operatorname{svd}(\mathcal{T}).
\end{equation}
We can obtain $\mathcal{U}= \left[{e}_1, e_2, \ldots, e_\mathcal{N}\right] $ as an orthonormal basis for the span of $\mathcal{T}$. The orthonormal basis $\mathcal{U}$ can be utilized to compute a semantic projection matrix $\mathbf{P} \in \mathcal{R}^{d \times d}$, which can be defined as follows:
\begin{equation}
\mathbf{P}=\mathcal{U}  \odot \mathcal{U}^{\top},
\end{equation}
where $\odot$ represents the matrix multiplication. Then, we can project the image and text features from the original space into the semantic subspace with the projection matrix $\mathbf{P}$. They can be calculated as follows:
\begin{equation}
\begin{split}
\widetilde{F_{I}(\mathbf{x}_{I})}=  F_{I}(\mathbf{x}_{I}) \odot \mathbf{P}, \\
\widetilde{F_{T}(\mathbf{x}_{T})}=  F_{T}(\mathbf{x}_{T}) \odot \mathbf{P},
\end{split}
\end{equation}
where $\widetilde{F_{I}(\mathbf{x}_{I})}$ and $\widetilde{F_{T}(\mathbf{x}_{T})}$ respectively represent the image feature embeddings and text feature embeddings after semantic projection. The loss function $J$, which is used to generate multimodal adversarial examples, can be rewritten as:
\begin{equation}
J = \frac{(F_I\left({\mathbf{x}}_{I}\right) \odot \mathbf{P}) \cdot (F_T\left({\mathbf{x}}_{T}\right)\odot \mathbf{P})}{d},
\end{equation}

\par The core of the semantic projection matrix is how to build an effective semantic corpus. In this paper, we construct the semantic corpus by creating a subset of the text-testing dataset, which is generated through random sampling of the complete testing dataset. To better understand the role of the semantic corpus, we analyze the impact of different sampling proportions $\mathcal{P}=\mathcal{N} / \mathcal{M}$, where $\mathcal{N}$ represents the number of samples used to generate the matrix and $\mathcal{M}$ represents the number in the entire text dataset. 
We adopt ALBEF as the source model to generate the multimodal adversarial examples. TCL, CLIP$_{\rm CNN}$, and CLIP$_{\rm ViT}$ are used as the target models to evaluate the adversarial transferability. We adopt varying proportions of text data from the Flickr30K dataset to generate the projection matrix, \textit{i.e.,} 20\%, 40\%, 60\%, 80\%, and 100\%. The experiment results are shown in Table~\ref{table:Proportion}. Analyses are summarized as follows. First, compared to the baseline DRA~\cite{gao2024boosting}, which calculates the loss value in the original contrast space, the proposed loss function incorporating varying proportions of text data demonstrates improved adversarial transferability. It demonstrates the effectiveness of the proposed semantic contrast space in improving adversarial transferability. Second, 
the projection matrices generated from text data with varying proportions yield different levels of improvement in adversarial transferability, with 40\% achieving the best improvement. Considering the computational expenses, we set the proportion to 40\% for the experiments. In the appendix, we provide a detailed algorithm for the proposed SA-AET.

\begin{table*}[t]
\centering
\caption{{Optimal Hyper-parameters Selection}. Attack Success Rate (\%) on different settings, {Top} for different values of $\kappa , \mu, \nu$ in Equation \ref{eq:text-deviate} and {Bottom} for different numbers of samples $m$.}
\label{tab:optimal_params}
\setlength\tabcolsep{0.4cm}
\begin{tabular}{@{}ccccccccc@{}}
\toprule
\multirow{2}{*}{Setting} & \multicolumn{2}{c}{ALBEF}       & \multicolumn{2}{c}{TCL}         & \multicolumn{2}{c}{CLIP$_\mathrm{ViT}$}   & \multicolumn{2}{c}{CLIP$_\mathrm{CNN}$}    \\ \cmidrule(l){2-9} 
& TR R@1         & IR R@1         & TR R@1         & IR R@1         & TR R@1        & IR R@1         & TR R@1         & IR R@1         \\ 
\midrule \midrule
 $\boldsymbol{[\kappa,\mu,\nu] = [0.2,0.0,0.8]}$ & \textbf{99.9} & 99.98 & 96.42 & 96.05 & 51.53 & 59.47 & 52.23 & 61.41 \\
 $\boldsymbol{[\kappa,\mu,\nu] = [0.2,0.2,0.6]}$ & \textbf{99.9} & 99.98 & 96.21 & \textbf{96.10} & 52.27 & 59.60 & 52.75 & 61.44 \\
$\boldsymbol{[\kappa,\mu,\nu] = [0.2,0.4,0.4]}$ & \textbf{99.9} & 99.98 & 96.00 & \textbf{96.10} & 52.52 & 59.60 & 52.87 & 61.13 \\
$\boldsymbol{[\kappa,\mu,\nu] = [0.2,0.6,0.2]}$ & \textbf{99.9} & 99.98 & 96.00 & 96.07 & 52.76 & 59.73 & 53.38 & 61.23 \\
$\boldsymbol{[\kappa,\mu,\nu] = [0.2,0.8,0.0]}$ & \textbf{99.9} & \textbf{100.0} & 95.89 & 96.05 & 53.13 & 59.95 & 53.64 & 61.17 \\
 $\boldsymbol{[\kappa,\mu,\nu] = [0.4,0.0,0.6]}$ & \textbf{99.9} & 99.98 & 96.84 & 96.02 & 53.50 & 62.34 & 55.43 & 63.88 \\
$\boldsymbol{[\kappa,\mu,\nu] = [0.4,0.2,0.4]}$ & \textbf{99.9} & \textbf{100.0} & 96.63 & 96.02 & 53.62 & 62.21 & 55.30 & 63.84 \\
 $\boldsymbol{[\kappa,\mu,\nu] = [0.4,0.4,0.2]}$ & \textbf{99.9} & \textbf{100.0} & \textbf{96.94} & 96.02 & 54.48 & 62.34 & 55.30 & 63.98 \\
$\boldsymbol{[\kappa,\mu,\nu] = [0.4,0.6,0.0]}$ & \textbf{99.9} & 99.98 & 96.84 & 96.05 & 54.60 & 62.21 & 55.30 & 63.95 \\
$\boldsymbol{[\kappa,\mu,\nu] = [0.6,0.0,0.4]}$ & \textbf{99.9} & 99.98 & 96.42 & 96.05 & 55.09 & 63.56 & 57.22 & 65.49 \\
\cellcolor{gray!40} $\boldsymbol{[\kappa,\mu,\nu] = [0.6,0.2,0.2]}$ & \cellcolor{gray!40} \textbf{99.9} & \cellcolor{gray!40} 99.98 & \cellcolor{gray!40} 96.42 & \cellcolor{gray!40} 96.02 & \cellcolor{gray!40} \textbf{55.71} & \cellcolor{gray!40} 63.69 & \cellcolor{gray!40} \textbf{57.22} & \cellcolor{gray!40} \textbf{65.63}\\
$\boldsymbol{[\kappa,\mu,\nu] = [0.6,0.4,0.0]}$ & \textbf{99.9} & 99.98 & 96.52 & 96.05 & 55.45 & 63.72 & 57.09 & 65.56 \\
 $\boldsymbol{[\kappa,\mu,\nu] = [0.8,0.0,0.2]}$ & \textbf{99.9} & \textbf{100.0} & 96.63 & 95.93 & 54.85 & 64.18 & 56.45 & 65.63 \\
$\boldsymbol{[\kappa,\mu,\nu] = [0.8,0.2,0.0]}$ & \textbf{99.9} & 99.98 & 96.73 & 95.83 & 54.97 & \textbf{64.47} & 56.32 & 65.49\\
\midrule \midrule
$\boldsymbol{m = 3}$ & \textbf{99.9} & \textbf{100.0} & 96.00 & 95.93 & 54.56 & 63.50 & 56.49 & 65.00\\
$\boldsymbol{m = 4}$ & \textbf{99.9} & \textbf{100.0} & 96.10 & 96.00 & 54.83 & 63.65 & 56.64 & 64.83\\ 
\cellcolor{gray!40} $\boldsymbol{m = 5}$ & \cellcolor{gray!40} \textbf{99.9} & \cellcolor{gray!40} 99.98 & \cellcolor{gray!40} \textbf{96.42} & \cellcolor{gray!40} 96.02 & \cellcolor{gray!40} \textbf{55.58} & \cellcolor{gray!40} 63.89 & \cellcolor{gray!40} 57.22 & \cellcolor{gray!40} \textbf{65.59}\\
$\boldsymbol{m = 6}$ & \textbf{99.9} & \textbf{100.0} & 95.89 & \textbf{96.21} & 54.23 & 63.56 & \textbf{57.24} & 65.19\\
$\boldsymbol{m = 7}$ & \textbf{99.9} & \textbf{100.0} & 96.33 & 96.17 & 55.58 & \textbf{63.95} & 56.98 & 65.15\\
\bottomrule
\end{tabular}
\end{table*}

\section{Experiments}
To evaluate the effectiveness of the proposed method, we conduct extensive experiments across different datasets and VLP models. First, we introduce the detailed experimental setup, including the VLP models, image-text pair benchmark datasets, the settings of adversarial attack, and the evaluation metrics in \textbf{Section}~\ref{Experimental_Settings}. Subsequently, we introduce the selection of the optimal hyper-parameters in \textbf{Section}~\ref{Hyper_parameters}. Then, we present the comparative experimental results with state-of-the-art methods on the cross-model adversarial transferability within the image-text retrieval task in \textbf{Section}~\ref{sec:comparison}. 
After that, we present the cross-task transferability of multimodal adversarial examples from the image-text retrieval task to other tasks, along with adversarial transferability testing of current state-of-the-art LLMs in \textbf{Section}~\ref{Cross_task} and \textbf{Section}~\ref{sec:Transferability_LLMs}, respectively. Additionally, we present the ablation study of the proposed method in \textbf{Section}~\ref{sec:ablation}. Finally, we introduce a knowledge transfer-based metric to offer an alternative perspective on the observed improvements in transferability in \textbf{Section}~\ref{sec:perfromance_analysis}.

\subsection{Experimental Settings}
\label{Experimental_Settings}
\noindent \textbf{Benchmark datasets.}  We adopt two widely used image-text pair datasets as the benchmark dataset to conduct experiments on the image-text retrieval task, \textit{i.e.,}  Flickr30K \cite{plummer2015flickr30k} and MSCOCO \cite{lin2014microsoft}. The Flickr30K dataset contains 31,783 images, each annotated with five captions. Similarly, the MSCOCO dataset includes 123,287 images, each accompanied by approximately five captions. For the visual grounding task, we adopt the RefCOCO+ \cite{yu2016modeling} dataset to conduct experiments. RefCOCO+ is a dataset with 141,564 referring expressions linked to 50,000 objects across 19,992 images from the MSCOCO collection. It is widely used to assess grounding models, focusing on how well they can locate objects described by natural language. For the Image Captioning task, which is another key challenge in vision-and-language research, we utilize the MSCOCO dataset. This dataset is widely used because it provides rich images paired with diverse captions.

\begin{table*}[t]
\caption{{Comparison with state-of-the-art methods on the Flickr30K dataset for the image-text retrieval task.} The source column shows VLP models we use to generate multimodal adversarial examples. For both Image Retrieval and Text Retrieval, we provide an R@1 attack success rate (\%).}
\begin{center}
\small
\renewcommand\arraystretch{1}
\setlength{\tabcolsep}{8pt}
    \resizebox{0.9\linewidth}{!}{
		\begin{tabular}{ @{\extracolsep{\fill}} c|c|cc|cc|cc|cc} 
        \toprule[0.3mm]
			& &  \multicolumn{2}{c}{\textbf{ALBEF}} & \multicolumn{2}{c}{\textbf{TCL}} & \multicolumn{2}{c}{\textbf{CLIP$_{\rm ViT}$}} & \multicolumn{2}{c}{\textbf{CLIP$_{\rm CNN}$}}  \\
			\cmidrule{3-10} 
			\multirow{-2}{*}{\textbf{Source}} &\multirow{-2}{*}{\textbf{Attack}} & {TR R@1} & {IR R@1} & {TR R@1} & {IR R@1} & {TR R@1} & {IR R@1} & {TR R@1} & {IR R@1}  \\
			\midrule
			\multirow{7}{*}{\rotatebox[origin=c]{0}{\textbf{ALBEF}}} 
            & PGD \cite{madry2017towards} &  93.74 &  94.43 & 24.03 & 27.90 & 10.67 & 15.82 & 14.05 & 19.11 \\
            & BERT-Attack \cite{li2020bertattack} &  11.57 &  27.46 & 12.64 & 28.07 & 29.33 & 43.17 & 32.69 & 46.11 \\
            & Sep-Attack \cite{zhang2022towards} &  95.72 &  96.14 & 39.30 & 51.79 & 34.11 & 45.72 & 35.76 & 47.92\\
            & Co-Attack \cite{zhang2022towards} &  97.08 &  98.36 & 39.52 & 51.24 & 29.82 & 38.92 & 31.29 & 41.99  \\
			& SGA \cite{lu2023setlevel} &  99.90 &  99.98 & 87.88 & 88.05 & 36.69 & 46.78 & 39.59 & 49.78\\
                & DRA \cite{gao2024boosting} &  99.90 &  99.93 & 91.57 & 91.17 & 46.26 & 56.80 & 49.55 & 59.01\\
                & \cellcolor{gray!40} \textbf{SA-AET (ours)} & \cellcolor{gray!40} \textbf{99.90} & \cellcolor{gray!40} \textbf{99.98} & \cellcolor{gray!40} \textbf{96.42} & \cellcolor{gray!40} \textbf{96.02} & \cellcolor{gray!40} \textbf{55.58} & \cellcolor{gray!40} \textbf{63.89} & \cellcolor{gray!40} \textbf{57.22} & \cellcolor{gray!40} \textbf{65.59}\\
			\midrule 
			\multirow{7}{*}{\rotatebox[origin=c]{0}{\textbf{TCL}}} 
            & PGD \cite{madry2017towards}  & 35.77 & 41.67 &  99.37 &  99.33 & 10.18 & 16.3 & 14.81 & 21.10 \\
            & BERT-Attack \cite{li2020bertattack} & 11.89 & 26.82 &  14.54 &  29.17 & 29.69 & 44.49 & 33.46 & 46.07 \\
            & Sep-Attack \cite{zhang2022towards} & 52.45 & 61.44 &  99.58 &  99.45 & 37.06 & 45.81 & 37.42 & 49.91\\
            & Co-Attack \cite{zhang2022towards} &  49.84 & 60.36 &  91.68 &  95.48 & 32.64 & 42.69 & 32.06 & 47.82\\
			& SGA \cite{lu2023setlevel} & 92.49 & 92.77 &  100.0 &  100.0 & 36.81 & 46.97 & 41.89 & 51.53\\
            & DRA \cite{gao2024boosting} & 95.20 & 95.58 &  100.0 &  99.98 & 47.24 & 57.28 & 52.23 & 62.23\\
            & \cellcolor{gray!40} \textbf{SA-AET (ours)}  & \cellcolor{gray!40} \textbf{98.85} & 
            \cellcolor{gray!40} \textbf{98.50} & 
            \cellcolor{gray!40} \textbf{100.0} & 
            \cellcolor{gray!40} \textbf{100.0} & 
            \cellcolor{gray!40} \textbf{56.20} & 
            \cellcolor{gray!40} \textbf{63.47} & 
            \cellcolor{gray!40} \textbf{59.77} & 
            \cellcolor{gray!40} \textbf{67.86}\\
			\midrule 
			\multirow{7}{*}{\rotatebox[origin=c]{0}{\textbf{CLIP$_{\rm ViT}$}}} 
            & PGD \cite{madry2017towards} & 3.13 & 6.48 & 4.43 & 8.83 &  69.33 &  84.79 & 13.03 & 17.43\\
            & BERT-Attack \cite{li2020bertattack}  & 9.59 & 22.64 & 11.80 & 25.07 & 28.34 & 39.08 & 30.40 & 37.43  \\
            & Sep-Attack \cite{zhang2022towards} & 7.61 & 20.58 & 10.12 & 20.74 &  76.93 &   87.44 & 29.89 & 38.32\\
            & Co-Attack \cite{zhang2022towards} & 8.55 & 20.18 & 10.01 & 21.29 &  78.53 &  87.50 & 29.50 & 38.49\\
			& SGA \cite{lu2023setlevel} & 22.42 & 34.59 & 25.08 & 36.45 &  100.0 &  100.0 & 53.26 & 61.10\\
                & DRA \cite{gao2024boosting} & 27.84 & 42.84 & 27.82 & 44.60 &  100.0 &  100.0  & 64.88 & 69.50\\
                & \cellcolor{gray!40} \textbf{SA-AET (ours)} &  \cellcolor{gray!40} \textbf{36.60} & 
                \cellcolor{gray!40} \textbf{50.44} & 
                \cellcolor{gray!40} \textbf{39.20} & 
                \cellcolor{gray!40} \textbf{51.10} & 
                \cellcolor{gray!40} \textbf{100.0} & 
                \cellcolor{gray!40} \textbf{100.0} & 
                \cellcolor{gray!40} \textbf{71.01} & 
                \cellcolor{gray!40} \textbf{74.10}\\
			\midrule 
			\multirow{7}{*}{\rotatebox[origin=c]{0}{\textbf{CLIP$_{\rm CNN}$}}} 
            & PGD \cite{madry2017towards} & 2.29 & 6.15 & 4.53 & 8.88 & 5.40 & 12.08 &  89.78 &  93.04\\
            & BERT-Attack \cite{li2020bertattack} & 8.86 & 23.27 & 12.33 & 25.48 & 27.12 & 37.44 & 30.40 & 40.10\\
            & Sep-Attack \cite{zhang2022towards} & 9.38 & 22.99 & 11.28 & 25.45 & 26.13 & 39.24 &  93.61 &  95.30\\
            & Co-Attack \cite{zhang2022towards} & 10.53 & 23.62 & 12.54 & 26.05 & 27.24 & 40.62 &  95.91 &  96.50\\
			& SGA \cite{lu2023setlevel} &  15.64 & 28.60 & 18.02 & 33.07 & 39.02 & 51.45 &  99.87 &  99.90\\
            & DRA \cite{gao2024boosting} & 19.50 & 34.59 & 21.60 & 37.88 & 48.47 & 59.12 &  99.87 &  99.90\\
            & \cellcolor{gray!40} \textbf{SA-AET (ours)} &  \cellcolor{gray!40} \textbf{23.98} & 
            \cellcolor{gray!40} \textbf{38.28} & 
            \cellcolor{gray!40} \textbf{27.29} &  
            \cellcolor{gray!40} \textbf{41.81} & 
            \cellcolor{gray!40} \textbf{54.11} & 
            \cellcolor{gray!40} \textbf{64.21} & 
            \cellcolor{gray!40} \textbf{100.0} & 
            \cellcolor{gray!40} \textbf{99.97}\\
			\bottomrule[0.3mm]

   \end{tabular}}
\end{center}
\label{tab:flickr}
\end{table*}

\begin{table*}[t]
\caption{{Comparison with state-of-the-art methods on the MSCOCO dataset for the image-text retrieval task.} The source column shows VLP models we use to generate multimodal adversarial examples. For both Image Retrieval and Text Retrieval, we provide an R@1 attack success rate (\%).}
\begin{center}
\small
\renewcommand\arraystretch{1}
\setlength{\tabcolsep}{8pt}
    \resizebox{0.9\linewidth}{!}{
		\begin{tabular}{ @{\extracolsep{\fill}} c|c|cc|cc|cc|cc} 
        \toprule[0.3mm]
			& &  \multicolumn{2}{c}{\textbf{ALBEF}} & \multicolumn{2}{c}{\textbf{TCL}} & \multicolumn{2}{c}{\textbf{CLIP$_{\rm ViT}$}} & \multicolumn{2}{c}{\textbf{CLIP$_{\rm CNN}$}}  \\
			\cmidrule{3-10} 
			\multirow{-2}{*}{\textbf{Source}} &\multirow{-2}{*}{\textbf{Attack}} & {TR R@1} & {IR R@1} & {TR R@1} & {IR R@1} & {TR R@1} & {IR R@1} & {TR R@1} & {IR R@1}  \\
			\midrule
			\multirow{7}{*}{\rotatebox[origin=c]{0}{\textbf{ALBEF}}} 
            & PGD \cite{madry2017towards} & 94.35 & 93.26 & 34.15 & 36.86 & 21.71 & 27.06 & 23.83 & 30.96 \\
            & BERT-Attack \cite{li2020bertattack} &  24.39 & 36.13 &  24.34 &  33.39 &  44.94 &   52.28 & 47.73 & 54.75 \\
            & Sep-Attack \cite{zhang2022towards} &  97.01 & 96.59 & 61.06 & 66.13 & 57.19 & 65.82 & 58.81 & 68.61\\
            & Co-Attack \cite{zhang2022towards} &  65.22 & 72.41 & 94.95 & 97.87 & 55.28 & 62.33 & 56.68 & 66.45\\
			& SGA \cite{lu2023setlevel} &  99.95 & 99.94 & 87.46 & 88.17 & 63.72 & 69.71 & 63.91 & 70.78 \\
                & DRA \cite{gao2024boosting} &  99.90&99.93&88.81&90.06&69.25&75.31&68.53&75.09\\
                & \cellcolor{gray!40} \textbf{SA-AET (ours)} & \cellcolor{gray!40} \textbf{100.0} & 
                \cellcolor{gray!40} \textbf{99.99} & 
                \cellcolor{gray!40} \textbf{97.28} & 
                \cellcolor{gray!40} \textbf{96.88} &  
                \cellcolor{gray!40} \textbf{76.57} & 
                \cellcolor{gray!40} \textbf{80.24} & 
                \cellcolor{gray!40} \textbf{76.17} & 
                \cellcolor{gray!40} \textbf{80.64}\\
			\midrule 
			\multirow{7}{*}{\rotatebox[origin=c]{0}{\textbf{TCL}}} 
            & PGD \cite{madry2017towards}  & 40.81 & 44.09 & 98.54 & 98.20 & 21.79 & 26.92 & 24.97 & 32.17\\
            & BERT-Attack \cite{li2020bertattack} &  35.32 & 45.92 & 38.54 &  48.48 &  51.09 & 58.80 & 52.23 &  61.26\\
            & Sep-Attack \cite{zhang2022towards} & 66.38 & 72.19 & 99.15 & 98.98 & 59.94 & 65.95 & 60.77 & 69.37 \\
            & Co-Attack \cite{zhang2022towards} &  49.84 & 60.36 &  91.68 &  95.48 & 32.64 & 42.69 & 32.06 & 47.82\\
			& SGA \cite{lu2023setlevel} & 92.70&92.99&100.0&100.0&59.79&65.31&60.52&67.34\\
            & DRA \cite{gao2024boosting} & 94.72&95.89&100.0&\textbf{100.0}&70.51&74.95&70.29&76.99\\
            & \cellcolor{gray!40} \textbf{SA-AET (ours)} & \cellcolor{gray!40} \textbf{97.78} & 
            \cellcolor{gray!40} \textbf{98.08} & 
            \cellcolor{gray!40} \textbf{100.0} & 
            \cellcolor{gray!40} 99.99 & 
            \cellcolor{gray!40} \textbf{76.12} & 
            \cellcolor{gray!40} \textbf{79.74} & 
            \cellcolor{gray!40} \textbf{75.89} & 
            \cellcolor{gray!40} \textbf{80.92}\\
			\midrule 
			\multirow{7}{*}{\rotatebox[origin=c]{0}{\textbf{CLIP$_{\rm ViT}$}}} 
            & PGD \cite{madry2017towards} & 10.26 & 13.69 & 12.72 & 15.81 & 82.91 & 90.51 & 21.62 & 28.78 \\
            & BERT-Attack \cite{li2020bertattack}  & 20.34 & 29.74 & 21.08 & 29.61 & 45.06 &  51.68 & 44.54 & 53.72\\
            & Sep-Attack \cite{zhang2022towards} & 25.91 & 36.84 & 28.20 & 38.47 & 88.36 & 97.09 & 47.57 & 57.79\\
            & Co-Attack \cite{zhang2022towards} & 26.35 & 36.69 & 88.78 & 96.72 & 28.23 & 38.42 & 47.36 & 58.45 \\
			& SGA \cite{lu2023setlevel} & 43.75&51.08&44.05&51.02&100.0&100.0&70.66&75.58\\
                & DRA \cite{gao2024boosting} & 52.69&61.50&51.88&61.06&100.0&100.0&80.18&84.11\\
                & \cellcolor{gray!40} \textbf{SA-AET (ours)} & \cellcolor{gray!40} \textbf{57.64} & 
                \cellcolor{gray!40} \textbf{66.88} & 
                \cellcolor{gray!40} \textbf{57.30} & 
                \cellcolor{gray!40} \textbf{65.16} & 
                \cellcolor{gray!40} \textbf{100.0} & 
                \cellcolor{gray!40} \textbf{100.0} & 
                \cellcolor{gray!40} \textbf{83.98} & 
                \cellcolor{gray!40} \textbf{86.72}\\
			\midrule 
			\multirow{7}{*}{\rotatebox[origin=c]{0}{\textbf{CLIP$_{\rm CNN}$}}} 
            & PGD \cite{madry2017towards} & 8.38 & 12.73 & 11.90 & 15.68 & 13.66 & 20.62 & 92.68 & 94.71\\
            & BERT-Attack \cite{li2020bertattack} & 23.38 & 34.64 & 24.58 & 29.61 & 51.28 &  57.49 & 54.43 &  62.17\\
            & Sep-Attack \cite{zhang2022towards} & 29.13 & 40.64 & 31.40 & 42.99 & 52.23 & 59.73 & 96.16 & 97.54\\
            & Co-Attack \cite{zhang2022towards} & 29.49 & 41.50 & 31.83 & 43.44 & 53.15 & 60.15 & 97.79 & 98.54\\
			& SGA \cite{lu2023setlevel} &  36.94&46.79&38.81&48.90&62.19&67.73&99.92&\textbf{99.97}\\
            & DRA \cite{gao2024boosting} & 41.40&52.25&43.62&54.15&70.43&74.14&99.80&99.92\\
            & \cellcolor{gray!40} \textbf{SA-AET (ours)} & \cellcolor{gray!40} \textbf{43.62} & 
            \cellcolor{gray!40} \textbf{55.19} & 
            \cellcolor{gray!40} \textbf{47.01} & 
            \cellcolor{gray!40} \textbf{57.39} & 
            \cellcolor{gray!40} \textbf{73.67} & 
            \cellcolor{gray!40} \textbf{76.90} & 
            \cellcolor{gray!40} \textbf{100.0} & 
            \cellcolor{gray!40} 99.92 \\
			\bottomrule[0.3mm]
	\end{tabular}}
\end{center}
\label{tab:coco}
\end{table*}

\begin{table*}[t]
\caption{ {Cross-Task Transferability.} 
We utilize ALBEF to generate multi-modal adversarial examples for attacking both Visual Grounding (VG) on the RefCOCO+ dataset and Image Captioning (IC) on the MSCOCO dataset. The baseline represents the performance of each task without any attack, where a lower value indicates better effectiveness of the adversarial attack for both tasks.}
\begin{center}
\small
\renewcommand\arraystretch{1}
\setlength{\tabcolsep}{16pt}
    \resizebox{0.9\linewidth}{!}{
		\begin{tabular}{ @{\extracolsep{\fill}} c|ccc|ccccc} 
        \toprule[0.3mm]
			& \multicolumn{3}{c}{\textbf{ITR $\rightarrow$ VG}} & \multicolumn{5}{c}{\textbf{ITR $\rightarrow$ IC}} \\
			\cmidrule{2-9}
			\multirow{-2}{*}{\textbf{Attack}} & {Val $\downarrow$} & {TestA $\downarrow$} &    {TestB $\downarrow$} & {B@4 $\downarrow$} & {METEOR $\downarrow$} & {ROUGE-L $\downarrow$} & {CIDEr $\downarrow$} & {SPICE $\downarrow$} \\
			\midrule
            Clean & 58.46 & 65.89 & 46.25 & 39.7 & 31.0 & 60.0 & 133.3 & 23.8\\
            SGA~\cite{lu2023setlevel} & 50.56 & 57.42 & 40.66 & 28.0 & 24.6 & 51.2 & 91.4 & 17.7 \\
            DRA~\cite{gao2024boosting} &  49.70  &  56.32 & 40.54 &  27.2 &  24.2 &  50.7 &  88.3 &  17.2\\
            \cellcolor{gray!40} \textbf{SA-AET (ours)} & 
            \cellcolor{gray!40} \textbf{47.44} & 
            \cellcolor{gray!40} \textbf{53.27} & 
            \cellcolor{gray!40} \textbf{38.58} & 
            \cellcolor{gray!40} \textbf{21.0} & 
            \cellcolor{gray!40} \textbf{20.5} & 
            \cellcolor{gray!40} \textbf{45.2} & 
            \cellcolor{gray!40} \textbf{65.7} & 
            \cellcolor{gray!40} \textbf{13.6} \\
			\bottomrule[0.3mm]
	\end{tabular}}
\end{center}
\label{tab:cross_task}
\end{table*}

\label{Cross_model}

\par \noindent \textbf{VLP models.} We employ two standard architectures of VLP models, \textit{i.e.,} fused and aligned VLP models, to carry out our experiments. In the case of aligned VLP models, CLIP \cite{radford2021learning} is used for the aligned VLP model. Specifically, we adopt different image encoders for CLIP, including ResNet-101 \cite{He_2016_CVPR} and ViT-B/16 \cite{dosovitskiy2010image}, \textit{i.e.,} CLIP$_{\rm CNN}$ and CLIP$_{\rm ViT}$, respectively. In the case of fused VLP models, 
we adopt ALBEF \cite{li2021align} and TCL \cite{yang2022vision} to conduct our experiments. ALBEF incorporates a 12-layer ViT-B/16 as its image encoder, alongside two 6-layer transformers dedicated to text and multimodal encoding. Although TCL adopts the same architectural framework, it distinguishes itself by employing different pre-training objectives.
\par \noindent \textbf{Adversarial attack settings.} Following previous works~\cite{lu2023setlevel,gao2024boosting,li2024improving}, for the text adversarial attack, we set the text perturbation bound $\epsilon_{T}=1$, with the word list $W$ containing 10 words. For the image adversarial attack, we set the image perturbation bound $\epsilon_{I}= 8/255$ under $L_{\infty}$ norm. Moreover, the number of iteration steps $T$ is set to 10, with each step size $\alpha = 2/255$. We follow the image augmentation techniques in SGA during its optimization process, enlarging the image dataset by resizing the original images to five different scales: $\left\{0.50, 0.75, 1.00, 1.25, 1.50 \right\}$, utilizing bicubic interpolation.
We compare the proposed SA-AET with the following baselines: (1) PGD~\cite{madry2017towards}, (2) Bert-Attack~\cite{li2020bertattack}, (3) Sep-Attack~\cite{zhang2022towards}, (4) Co-Attack~\cite{zhang2022towards}, (5) SGA~\cite{lu2023setlevel}, and (6) DRA~\cite{gao2024boosting}.

\par \noindent \textbf{Evaluation metrics.} We adopt the Attack Success Rate (ASR) at the top-1 rank (R@1) as the metric to evaluate the adversarial transferability. It represents the percentage of successful attacks among all generated adversarial examples. A higher ASR indicates a more effective attack with higher transferability. 
%
\subsection{Optimal Hyper-Parameters Selection}
\label{Hyper_parameters}
The number of samples $m$ taken from the sub-triangle-$\mathcal{A}$ of the adversarial trajectory and scaling factors in Equation \ref{eq:text-deviate} are the hyper-parameters. We conduct a series of experiments to determine the optimal value of hyper-parameters and analyze their impact on the overall performance of the proposed method. Specifically, we employ ALBEF to generate multimodal adversarial examples on the Flickr30K dataset and evaluate their transferability to three other VLP models.

\par \noindent \textbf{Number of samples $m$ taken from the sub-triangle-$\mathcal{A}$ of the adversarial trajectory:}
The results are shown at the bottom of Table \ref{tab:optimal_params}. It can be observed that 
 transfer effects become apparent for both the Image Retrieval and Text Retrieval tasks when the sample size reaches 5. Beyond this point, as the sample size increases, transferability shows fluctuations without any significant enhancement. Considering the balance between transfer effects and computational efficiency, a sample size of 5 is the most effective hyperparameter. 
\par \noindent \textbf{Scaling factors $\kappa,\mu,\nu$ in adversarial text generation:}
In Equation \ref{eq:text-deviate}, we introduce three hyper-parameters that represent the weights of the clean image, the previous adversarial example, and the current adversarial example, respectively. We stipulate that the proportion of adversarial images cannot be zero, \ie, $\mu + \nu \neq 0$. Under the above conditions, we assign $\kappa,\mu,\nu$ to 14 different sets of values. Using ALBEF as the surrogate model, we generate adversarial examples and evaluate their transferability to the remaining three models. The detailed results are shown at the top of Table \ref{tab:optimal_params}. Adversarial transferability from ALBEF to TCL shows minimal variation across settings (around 1\%), while transferability to CLIP models varies significantly, with a performance gap between 4.18\% and 5.0\%. Thus, our hyper-parameter selection focuses on maximizing transferability to CLIP models. Increasing $\kappa$, which introduces a higher proportion of clean images, initially boosts transferability, but an excessive $\kappa$ reduces the adversarial ratio, decreasing overall effectiveness. Based on Table~\ref{tab:optimal_params}, we select $\boldsymbol{[\kappa,\mu,\nu]} = [0.6,0.2,0.2]$ for optimal transferability. 

\subsection{Cross-model Adversarial Transferability}
\label{sec:comparison}
We adopt four widely used VLP models, \textit{i.e.,} ALBEF, TCL, CLIP$_{\rm ViT}$, and CLIP$_{\rm CNN}$, to conduct comparative experiments on the image-text retrieval task. Specifically, we adopt one of four VLP models to generate multimodal adversarial examples and then use the other three VLP models to evaluate the transferability of adversarial examples. We compare the proposed SA-AET with a series of advanced adversarial attack methods for VLP models on the Flickr30K and MSCOCO datasets. 
\par \noindent \textbf{Performance on the Flickr30K dataset.} 
Comparative experimental results on the Flickr30K dataset are shown in Table~\ref{tab:flickr}. It can be observed that our SA-AET, along with SGA and DRA, consistently outperforms other adversarial attack methods under all white box attack scenarios. Regardless of whether using TR or IR, ASR consistently surpasses 99.8\%. Compared with previous works, the proposed SA-AET achieves the best performance of adversarial transferability across different VLP models under all adversarial attack scenarios. For example, 
when ALBEF is used as the source model to generate multimodal adversarial examples, the advanced SGA method achieves an ASR of 87.88\% on TCL, 36.69\% on CLIP$_{\rm ViT}$, and 39.59\% on CLIP$_{\rm CNN}$ for text retrieval tasks. For image retrieval tasks, it attains an ASR of 88.05\% on TCL, 46.78\% on CLIP$_{\rm ViT}$, and 49.78\% on CLIP$_{\rm CNN}$. In contrast, the proposed SA-AET method outperforms SGA with an ASR of 96.42\% on TCL, 55.58\% on CLIP$_{\rm ViT}$, and 57.22\% on CLIP$_{\rm CNN}$ for text retrieval tasks. For image retrieval tasks, SA-AET achieves an ASR of 96.02\% on TCL, 63.89\% on CLIP$_{\rm ViT}$, and 65.59\% on CLIP$_{\rm CNN}$.
Moreover, when applying ALBEF to target CLIP$_{\rm CNN}$, the proposed SA-AET improves the ASR by 7.67\% for TR R@1 and 6.58\% for IR R@1, compared to DRA. The results demonstrate the effectiveness of our proposed method in improving the transferability of multimodal adversarial examples on the Flickr30K dataset.
\par \noindent \textbf{Performance on the MSCOCO dataset.} Comparative experimental results on the MSCOCO are shown in Table~\ref{tab:coco}.
It can also be observed that compared with previous works, the proposed method achieves the best performance in improving the adversarial transferability for VLP models. Specifically, 
when using CLIP$_{\rm ViT}$ to generate adversarial examples, the advanced SGA method demonstrates an ASR of 43.75\% on ALBEF, 44.05\% on TCL, and 70.66\% on CLIP$_{\rm CNN}$ for text retrieval tasks. For image retrieval tasks, it achieves an ASR of 51.08\% on ALBEF, 51.02\% on TCL, and 75.58\% on CLIP$_{\rm CNN}$. Conversely, the proposed SA-AET method surpasses SGA, achieving an ASR of 57.64\% on ALBEF, 57.3\% on TCL, and 83.98\% on CLIP$_{\rm CNN}$ for text retrieval tasks. In image retrieval tasks, SA-AET records ASRs of 66.88\% on ALBEF, 65.16\% on TCL, and 86.72\% on CLIP$_{\rm CNN}$. Additionally, when applying CLIP$_{\rm ViT}$ to target TCL, the proposed SA-AET improves the ASR by 5.42\% for TR R@1 and 4.1\% for IR R@1, compared to DRA. The results highlight the effectiveness of our proposed method in enhancing the transferability of multimodal adversarial examples on the MSCOCO dataset.

\begin{figure*}[t]
    \centering
    \includegraphics[width=1.0\linewidth]{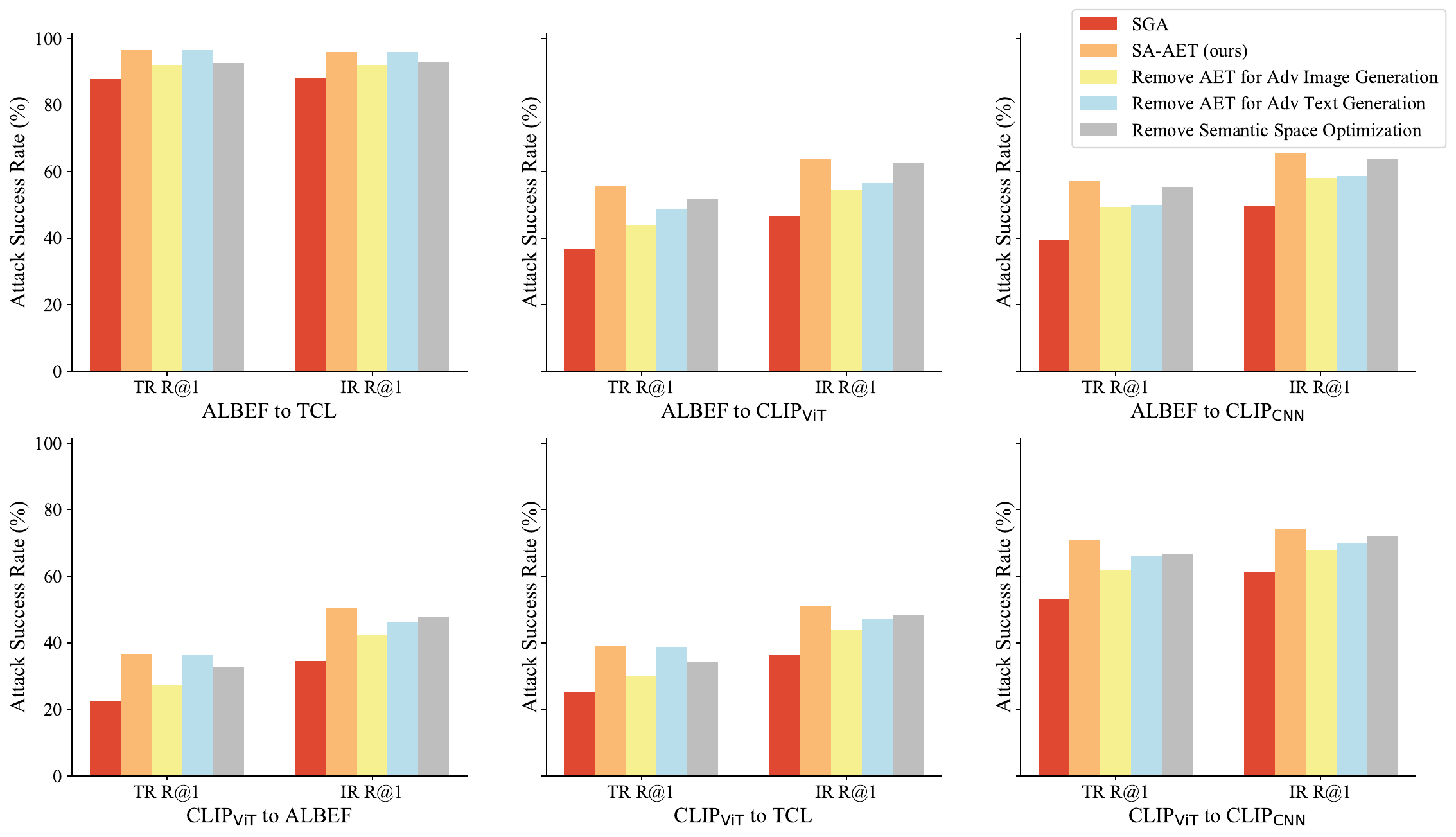}
    \caption{{Ablation Study.} {`Remove AET for Adv Image Generation'} refers to eliminating the proposed adversarial evolution triangle for adversarial image generation. {`Remove AET for Adv Image Generation'} refers to eliminating the proposed adversarial evolution triangle for adversarial text generation. {`Remove Semantic Space Optimization'} refers to eliminating the proposed semantic contrast space. }
    \label{fig:ablation}
\end{figure*}

\subsection{Cross-task Adversarial Transferability}
\label{Cross_task}
\begin{figure}[t]
    \centering
    \includegraphics[width=\linewidth]{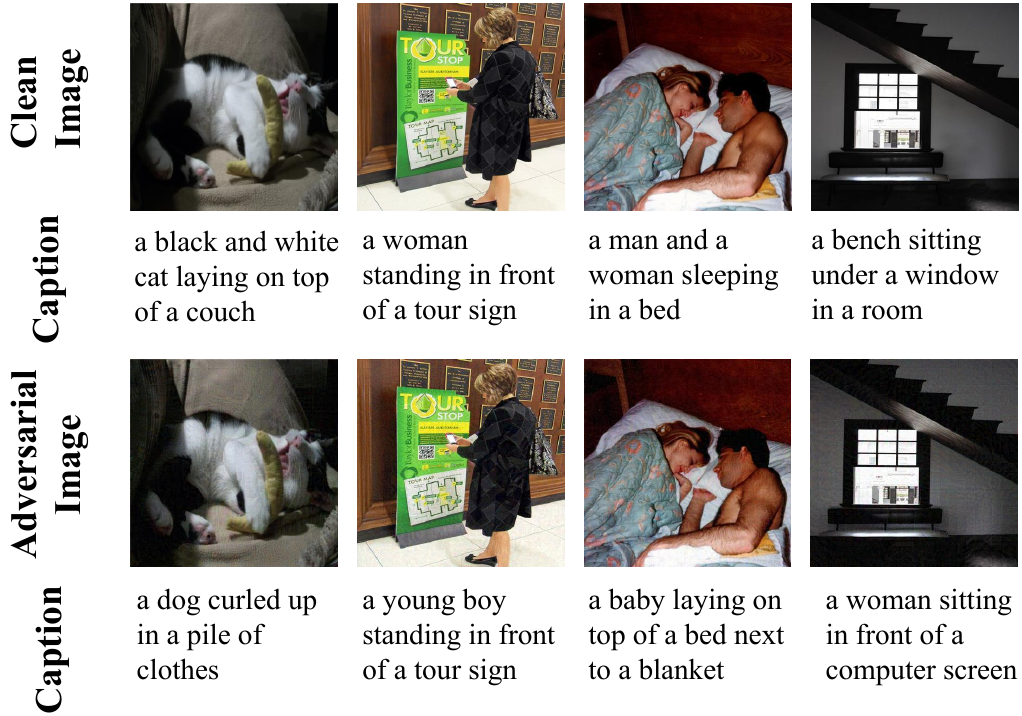}
    \caption{ {Visualization on Image Captioning Task.} We use the ALBEF model, pre-trained on Image Text Retrieval (ITR) task, to generate adversarial images on the MSCOCO dataset and use the BLIP \cite{li2022blip} model for Image Captioning on both clean images and adversarial images, respectively.}
    \label{fig:IC_Visualization}
\end{figure}
\begin{figure}[t]
    \centering
    \includegraphics[width=\linewidth]{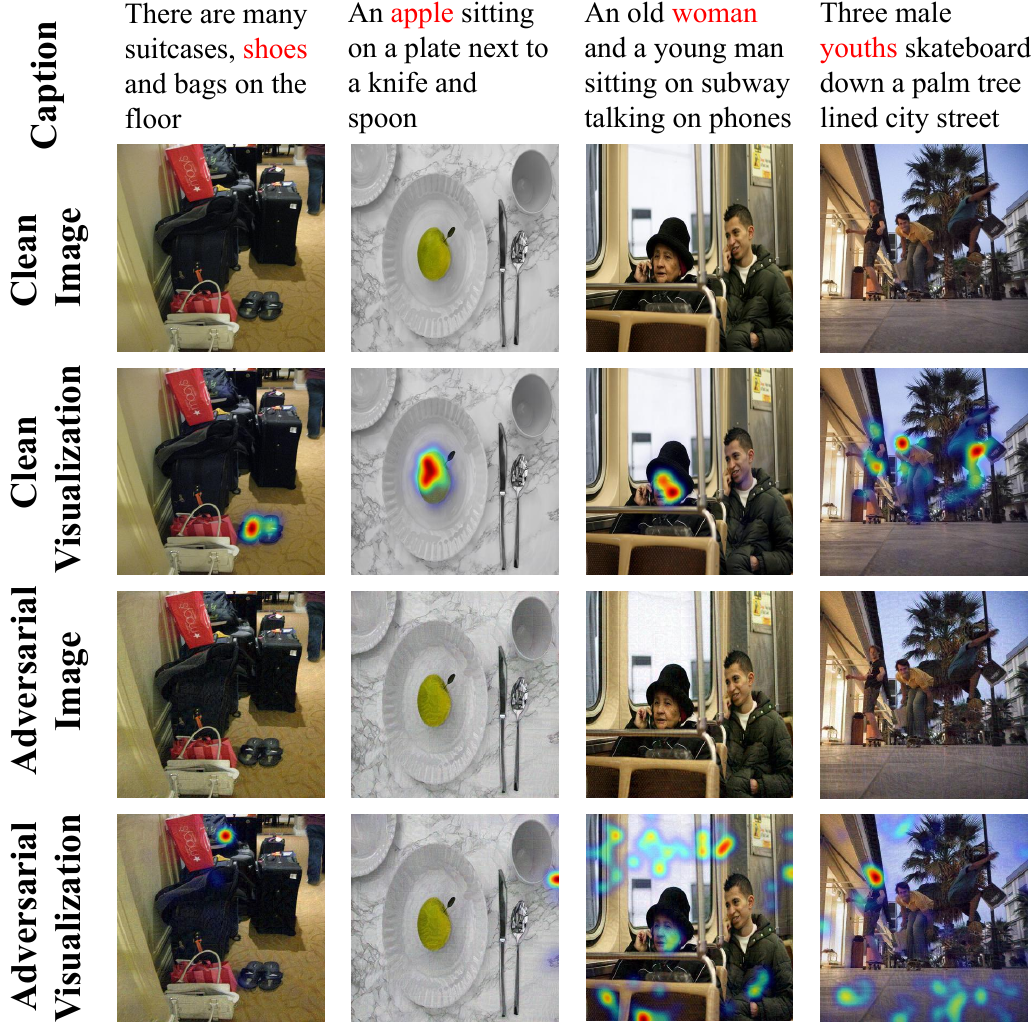}
    \caption{{Visualization on Visual Grounding Task.} We use the ALBEF model, pre-trained on the ITR task, to generate adversarial images on the RefCOCO+ dataset and use the same model, pre-trained on Visual Grouding (VG) task, to localize the regions corresponding to red words on both clean images and adversarial images, respectively.}
    \label{fig:VG_Visualization}
\end{figure}
We not only evaluate the transferability of multimodal adversarial examples generated by our proposed method across different VLP models but also conduct experiments to evaluate its effectiveness in transferring across diverse V+L tasks. Specifically, we generate adversarial examples for the Image-Text Retrieval (ITR) task and evaluate them on Visual Grounding (VG) and Image Captioning (IC) tasks. As evident from Table \ref{tab:cross_task} and visual results in Figures \ref{fig:IC_Visualization} and \ref{fig:VG_Visualization}, the adversarial examples generated for ITR demonstrate transferability, successfully impacting both VG and IC tasks.  Compared to SGA and DRA, adversarial examples generated by SA-AET lead to a substantial decline in metrics for both VG and IC tasks, with a notable drop of 25.7 and 22.6 points in the CIDEr score on the IC task, respectively. This demonstrates the superior effectiveness of SA-AET in transfer attacks, emphasizing its cross-task transferability. Moreover, our transferability consistently surpasses that of all other methods.

\subsection{Adversarial Transferability in MLLMs}
\label{sec:Transferability_LLMs}
\begin{table}[t]
\caption{{Transferability performance on the advanced commercial multimodal large Language Models.}}

\begin{center}
\small
\renewcommand\arraystretch{1}
\setlength{\tabcolsep}{8pt}
    \resizebox{0.8\linewidth}{!}{
		\begin{tabular}{ @{\extracolsep{\fill}} c|c c c} 
        \toprule[0.3mm]
			\textbf{Attack} & \textbf{GPT-4 R@1} & \textbf{Claude-3 R@1} & \textbf{Qwen-VL R@1} \\
			\midrule
                No Attack & 0.0  & 0.0  & 0.0 \\
                SGA \cite{lu2023setlevel} & 6.0  & 18.0  & 6.0 \\
                DRA \cite{gao2024boosting} & 16.0  & 22.0   & 15.0  \\
                \cellcolor{gray!40} SA-AET (ours) & 
                \cellcolor{gray!40} \textbf{18.0}  & 
                \cellcolor{gray!40} \textbf{24.0}  & 
                \cellcolor{gray!40} \textbf{16.0} \\
			\bottomrule[0.3mm]
	\end{tabular}}
\end{center}
\label{tab:rebuttal_llms}
\end{table}

\begin{figure}[t]
    \centering
    \includegraphics[width=\linewidth]{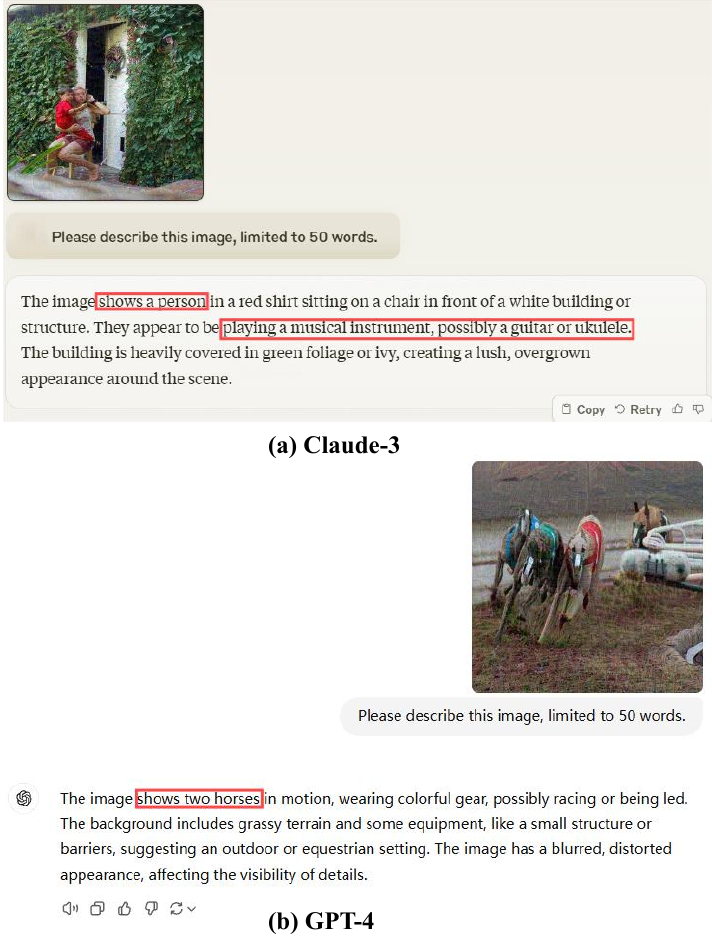}
    \caption{{Adversarial Transferability on Multimodal Large Language Models (MLLMs).} We input adversarial images along with the query ``Please describe this image, limited to 50 words." into MLLMs. The upper section represents the results from Claude-3, while the lower section shows the results from GPT-4. Incorrect descriptions are highlighted in red boxes.}
    \label{fig:LLMs_vis}
\end{figure}
Multimodal Large Language Models (MLLMs) have recently gained significant attention for their wide range of applications, demonstrating remarkable versatility and effectiveness in handling complex tasks. To further investigate the adversarial robustness of MLLMs, we conduct a series of extension experiments. Using ALBEF as a surrogate model, we generate adversarial examples under constrained perturbation settings, with a magnitude of 16/255 and a single-step perturbation of 0.5/255, iterating over 100 steps. These adversarial examples are then evaluated on advanced MLLMs (\textit{e.g.}, GPT-4~\cite{achiam2023gpt}, Claude-3~\cite{anthropic2024}) by prompting them with the query ``Describe this image". As shown in Figure~\ref{fig:LLMs_vis}, it illustrates that adversarial examples generated by our method are capable of deceiving state-of-the-art MLLMs, leading them to produce incorrect responses. In addition to these qualitative results on MLLMs, we further extend our experiments by randomly selecting 100 images from the Flick30K dataset. We generate adversarial images using SGA, DRA, and our method SA-AET, with ALBEF as the surrogate model, under the same perturbation setting (16/255). These adversarial images are then fed into MLLMs, where the query is prompted with ``Please describe this image, limited to 50 words". We rank the generated descriptions against 100 irrelevant texts based on CLIP similarity. If a description does not rank first, we consider it a successful attack. The results are shown in Table~\ref{tab:rebuttal_llms}. It can be observed that our SA-AET achieves the best transferability performance among advanced commercial MLLMs.

\subsection{Ablation Study}
\label{sec:ablation}
Our proposed method builds upon the SGA~\cite{lu2023setlevel} method and introduces three significant improvements: 1) the adversarial evolution triangle for adversarial image generation, 2) the adversarial evolution triangle for adversarial text generation, and 3) the contrast space optimization for semantic-aligned adversarial evolution triangle.
To evaluate the contribution of each element, we conduct ablation studies on the image-text retrieval (ITR) task and evaluate the transferability of multimodal adversarial examples from two different architectures: the fused and aligned architectures, corresponding to ALBEF and CLIP$_{\rm ViT}$, respectively. We then test how well these adversarial examples generalize to other VLP models. The results of the ablation studies can be seen in Figure~\ref{fig:ablation}. In the ablation studies, we compare five methods. The first is SGA \cite{lu2023setlevel}, which serves as our baseline, followed by SA-AET, representing our complete method. The other three methods are derived from SA-AET by sequentially removing each of the three key contributions we introduce. \textbf{`Remove AET for Adv Image Generation'} refers to eliminating the proposed adversarial evolution triangle for adversarial image generation. \textbf{`Remove AET for Adv Text Generation'} refers to eliminating the proposed adversarial evolution triangle for adversarial text generation. \textbf{`Remove Semantic Space Optimization'} refers to eliminating the proposed semantic contrast space. It is evident from the six sets of transfer attack experiments on ALBEF and CLIP$_{\rm ViT}$ that removing any of the improvements from SA-AET results in weaker transfer attacks compared to the full SA-AET method. It demonstrates the effectiveness of each element of our SA-AET. 

\subsection{Performance Analysis}
\label{sec:perfromance_analysis}
\begin{figure}[t]
    \centering
    \includegraphics[width=1.0\linewidth]{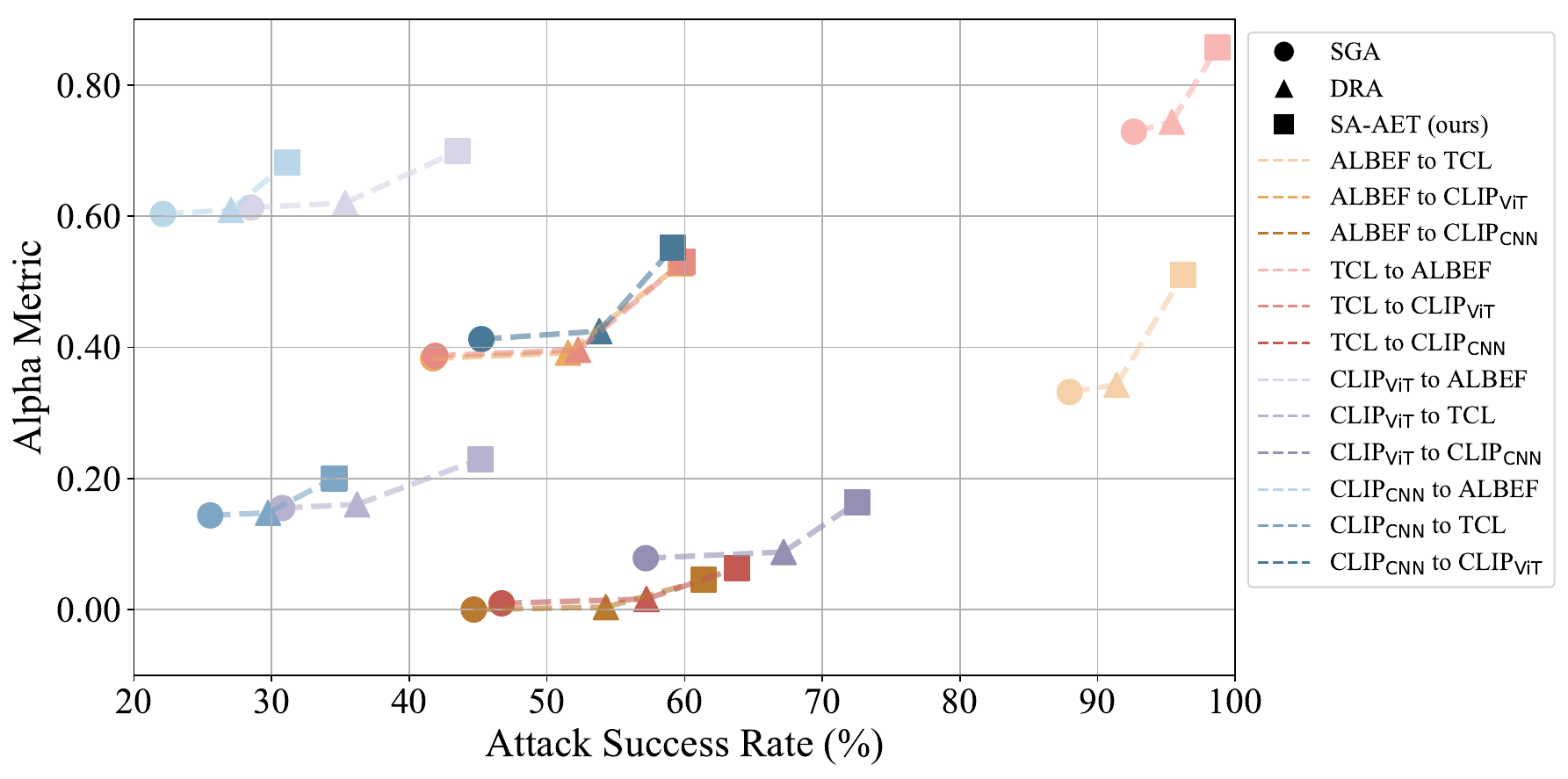}
    \caption{{Attack Success Rate vs. Alpha Metric in Transfer Attacks on VLP Models.} The horizontal axis represents the attack success rate of transfer attacks on the target model, while the vertical axis represents the alpha metric.}
    \label{fig:alpha_metric}
\end{figure}
In addition to examining the transfer attack success rates of multimodal adversarial examples generated by various VLP models, we introduce a knowledge transfer-based metric proposed by Liang \etal~\cite{liang2021uncovering} to offer an alternative perspective on the observed improvements in transferability and to provide insights into the underlying reasons for these enhancements, from the baseline method, SGA \cite{lu2023setlevel}, to the advanced approaches, DRA \cite{gao2024boosting} and SA-AET. The metric, known as the alpha metric, represents how often the adversarial attack transfers, which can be calculated as follows:

\begin{equation}
\label{eq:alpha}
\frac{J\left(F_{I}\left(\mathbf{x}_{I}+\epsilon^{sur}\right),F_T\left(\tilde{\mathbf{x}}^{sur}_{T}\right)\right)}{J\left(F_{I}\left(\mathbf{x}_{I}+\epsilon^{tar}\right),F_T\left(\tilde{\mathbf{x}}^{tar}_{T}\right)\right)}.
\end{equation}

In the above formula, $F_{I}$ and $F_{T}$ represent the image and text encoders of the target VLP model for transfer attacks, respectively, while $J\left(\cdot \right)$ denotes the loss function of the VLP model. $\epsilon^{sur}$ and $\tilde{\mathbf{x}}^{sur}_{T}$ represent the adversarial image perturbation and adversarial text generated on the surrogate model, respectively, while $\epsilon^{tar}$ and $\tilde{\mathbf{x}}^{tar}_{T}$ denote the adversarial image perturbation and adversarial text generated on the target model. In this context, the alpha metric can be viewed as the difference in loss between the multimodal adversarial examples generated on the surrogate model and its own multimodal adversarial examples. To measure the alpha metric, we select the Flickr30k dataset and utilize 12 sets of transfer attacks from four chosen VLP models, calculating the metric according to Formula~(\ref{eq:alpha}). We present the final results in Figure~\ref{fig:alpha_metric}. Notably, the $x$-axis represents the attack success rates of the three methods: SGA, DRA, and SA-AET. These success rates are obtained by averaging the TR R@1 and IR R@1 metrics. The alpha metric is calculated as the average across all Flickr30k data.

Based on the results presented in Figure \ref{fig:alpha_metric}, we find that SA-AET consistently outperforms both SGA and DRA, demonstrating significantly higher attack success rates and alpha metrics across all 12 sets of transfer attacks. Furthermore, the DRA method exhibits favorable performance, with both the attack success rates and alpha metrics for DRA and SA-AET showing substantial improvements over SGA. This finding suggests that our contribution to increasing adversarial transferability in both DRA and SA-AET is effective. While the difference in alpha metrics between DRA and SGA is relatively modest, there is a significant gap between SA-AET and the other two methods. The transfer attack success rates of DRA compared to SGA have already reached a high level, making further improvements challenging. However, by introducing semantic-aligned adversarial sub-triangles, we achieve a notable increase in the alpha metric. This indicates that we enhance how often the adversarial attack transfers, thereby further improving adversarial transferability.

\section{Conclusion}
In this paper, we focus on generating highly transferable adversarial examples (AEs) for vision-language pre-training (VLP) models. We find that previous methods focus only on augmenting image-text pairs around current adversarial examples, which offers limited improvements in transferability. To overcome this limitation, we propose leveraging the intersection regions along the adversarial trajectory during optimization, specifically by sampling from adversarial evolution triangles that are composed of clean, historical, and current adversarial examples. It significantly enhances the diversity of AEs, leading to improved transferability. Additionally, we tackle the issue of feature-matching distortion caused by redundant image features in VLP models. Adversarial examples generated in the original image-text feature contrast space are highly dependent on the victim model. We propose a semantic image-text feature contrast space to mitigate this, projecting the original features into a semantically aligned subspace. This projection reduces feature redundancy and enhances the transferability of adversarial examples. Our theoretical analysis validates the effectiveness of the proposed adversarial evolution triangles, and extensive experiments across various datasets and models demonstrate that our method outperforms state-of-the-art adversarial attack methods for VLP models. 

\bibliographystyle{plain}
\bibliography{main}

\vfill

\newpage

\onecolumn
\appendix


\section*{Detailed Algorithm}
\label{sec:Algorithm}
\begin{figure}[!h]
    \removelatexerror
    \begin{algorithm}[H]
        \caption{Attack formulation.}
        \begin{algorithmic}[1]
            \REQUIRE Image encoder $F_I$, Text encoder $F_T$, Image-caption pair $(\mathbf{x}_{I}, \mathbf{x}_{T})$, iteration steps $T$,  number of samples $m$, image maximal perturbation
            bound $\epsilon_{I}$, text maximal number
            of changeable words $\epsilon_{T}$,  single-step perturbation size $\alpha$, image augmentation scale set $\mathcal{I}=\left\{I_1,I_2,...,I_n\right\}$ 
            \ENSURE Adversarial Image $\tilde{\mathbf{x}}_{I}^T$, Adversarial Text $\tilde{\mathbf{x}}_{T}$
            \STATE /*Follow SGA and get $\tilde{\mathbf{x}}_{I}^{0}$ and $\tilde{\mathbf{x}}_{I}^{1}$ *\//
            \STATE $\tilde{\mathbf{x}}_{I}^{0} = \mathbf{x}_{I} + \epsilon_{I} \cdot N(0,1)$
            \STATE Construct the image augmentation set $\left\{\tilde{\mathbf{x}}_{I}^{0 1}, \tilde{\mathbf{x}}_{I}^{0 2}, \ldots, \tilde{\mathbf{x}}_{I}^{0 n}\right\}$ for $\tilde{\mathbf{x}}_{I}^{0}$ using the scales in $\mathcal{I}$.
            \STATE $\tilde{\mathbf{x}}_{I}^{1}=\underset{\mathbf{x}_{I},\epsilon_{I}}{\Pi}\left(\tilde{\mathbf{x}}_{I}^{0}+\alpha \cdot \operatorname{sign}\left(\frac{\nabla_{\mathbf{x}_{I}} \sum_{j=1}^n J(\tilde{\mathbf{x}}_{I}^{0 j},\mathbf{x}_{T})}{\left\|\nabla_{\mathbf{x}_{I}} \sum_{j=1}^n J(\tilde{\mathbf{x}}_{I}^{0 j},\mathbf{x}_{T})\right\|}\right)\right).$
            \STATE /* Adversarial Image Generation *\//
            \FOR{$i=1$ to $T-1$}
                \STATE Sample n instances from adversarial evolution sub-triangle-$\mathcal{A}$ and build set $\mathcal{S}^{i}=\left\{s_1^{i}, s_2^{i}, \ldots, s_m^{i}\right\}$.
                \STATE $s_k=\lambda \cdot \mathbf{x}_{I}+\beta \cdot \tilde{\mathbf{x}}_{I}^{i-1}+\gamma \cdot \tilde{\mathbf{x}}_{I}^{i}, \text { where } \lambda+\beta+\gamma=1.0 \text{ and } \lambda > \beta > \gamma.$
                \STATE Build perturbation set $\epsilon^{i} = \left\{\epsilon^{i}_{1}, \epsilon^{i}_{2}, \ldots, \epsilon^{i}_{m}\right\}$.
                \STATE $\epsilon^{i}_{k}=\alpha \cdot \operatorname{sign}\left(\frac{\nabla_{s^{i}_{k}} J\left(s^{i}_{k}, \mathbf{x}_{T}\right)}{\left\|\nabla_{s^{i}_{k}} J\left(s^{i}_{k}, \mathbf{x}_{T}\right)\right\|}\right).$
                \STATE Text-guided augmentation selection: obtain the optimal sample $s^{i}_{o}$.
                \STATE $o = \underset{\epsilon^{i}_{k}\in\epsilon^{i}}{\arg \max} J\left(\tilde{\mathbf{x}}_{I}^{i}+\epsilon^{i}_{k},\mathbf{x}_{T}\right).$
                \STATE Construct the image augmentation set $\mathcal{S}^{i}_{O} = \left\{s^{i1}_{o},s^{i2}_{o}, \ldots, s^{in}_{o} \right\}$ for $s^{i}_{o}$ using the scales in $\mathcal{I}$.
                \STATE $\tilde{\mathbf{x}}_{I}^{i+1}=\underset{\mathbf{x}_{I},\epsilon_{I}}{\Pi}\left(\tilde{\mathbf{x}}_{I}^{i}+\alpha \cdot \operatorname{sign}\left(\frac{\nabla_{s^{i}_{o}} \sum_{j=1}^n J\left(s^{ij}_{o}, \mathbf{x}_{T}\right)}{\left\|\nabla_{s^{i}_{o}} \sum_{j=1}^n J\left(s^{ij}_{o}, \mathbf{x}_{T}\right)\right\|}\right)\right).$
            \ENDFOR
            \STATE /* Adversarial Text Generation *\//
            \STATE $\tilde{\mathbf{x}}_{T}  =  \underset{\tilde{\mathbf{x}}_{T} \in B\left[{\mathbf{x}}_{T}, \epsilon_t\right]}{\arg \max } \left(\kappa \cdot J\left({\mathbf{x}}_{I}, \tilde{\mathbf{x}}_{T}\right) +\mu \cdot J\left(\tilde{\mathbf{x}}_{I}^{T-1},\tilde{\mathbf{x}}_{T}\right) +\nu \cdot J\left(\tilde{\mathbf{x}}_{I}^{T}), \tilde{\mathbf{x}}_{T}\right)\right).$
        \end{algorithmic}
    \end{algorithm}
    \vspace{-5mm}
\end{figure}

\section*{Proof of Theorem \ref{thm:goldbach}}
\label{sec:proof}
We first consider the following proposition to prove Theorem \ref{thm:goldbach}.
\begin{proposition}[Update Rules]
    The adversarial perturbation generated by the proposed method at $t$-step ($t\geq 2$) is
    \begin{equation}
        \begin{aligned}
             & \boldsymbol{g}_t &=&\ \  a_t\cdot \boldsymbol{g}+b_t\cdot\boldsymbol{Hg},\\
             & \boldsymbol{\delta}_{t} &=&\ \ \sum_{i=1}^{t} \boldsymbol{g}_i = c_t\cdot \boldsymbol{g} + d_t\cdot\boldsymbol{Hg},\\
        \end{aligned}
    \end{equation}
    where $\boldsymbol{g}$ is the gradient and $\boldsymbol{H}$ is the Hessian matrix of the loss function $L$ concerning $\boldsymbol{x}$, and 
    \begin{equation}
        \begin{aligned}
            & a_t &=&\ \ \ 1,\\
            & b_t &=&\ \ \ \beta\cdot(t-2)+\gamma\cdot(t-1),\\
            & c_t &=&\ \ \ t,\\
            & d_t &=&\ \ \ \sum_{i=2}^{t}\beta(i-2)+\gamma(i-1) = \frac{(t-1)(t-2)}{2}\beta+\frac{t(t-1)}{2}\gamma.
        \end{aligned}
    \end{equation}
    Then the update rule of \textbf{SAG} will be
    \begin{equation}
        \begin{aligned}
             & \boldsymbol{g}'_t &=&\ \  e_t\cdot \boldsymbol{g}+f_t\cdot\boldsymbol{Hg},\\
             & \boldsymbol{\zeta}_{t} &=&\ \ \sum_{i=1}^{t} \boldsymbol{g}'_i = h_t\cdot \boldsymbol{g} + l_t\cdot\boldsymbol{Hg},\\
        \end{aligned}
    \end{equation}
    where
    \begin{equation}
        \begin{aligned}
            & e_t &=&\ \ \ 1,\\
            & f_t &=&\ \ \ t-1,\\
            & h_t &=&\ \ \ t,\\
            & l_t &=&\ \ \ \frac{1}{2}\cdot t(t-1).
        \end{aligned}
    \end{equation}
\end{proposition}

\begin{proof}
    We use mathematical induction to complete the proof.
    \begin{enumerate}
        \item When $t=2$, by \eqref{eq:delta_update} and \eqref{eq:g_update}, we have 
        \begin{equation}
            \begin{aligned}
                & a_2 &=&\ \ \ 1,\\
                & b_2 &=&\ \ \ \gamma,\\
                & c_2 &=&\ \ \ 2,\\
                & d_2 &=&\ \ \ \gamma.
            \end{aligned}
        \end{equation}
        \item Assuming that the update rule holds when $m = t \geq 2$, and when $m = t+1$ we have
        \begin{equation}
            \begin{aligned}
                & \boldsymbol{g}_{t+1} &=&\ \ g\left(\boldsymbol{x}+\beta\cdot\boldsymbol{\delta}_{t-1}+\gamma\cdot \boldsymbol{\delta}_{t}\right)\\
                & &=&\ \sum_{n=0}^{\infty}\frac{g^{(n)}(\boldsymbol{x})}{n!}\left(\boldsymbol{x}+\beta\cdot\sum_{i=1}^{t-1}\boldsymbol{g}_i+\gamma\cdot\sum_{i=1}^{t}\boldsymbol{g}_i-\boldsymbol{x}\right)^n\\
                & &\approx&\ \ \boldsymbol{g} +\beta\cdot\sum_{i=1}^{t-1}\boldsymbol{H}\boldsymbol{g}_i+\gamma\cdot\sum_{i=1}^{t}\boldsymbol{H}\boldsymbol{g}_i\\
                & &=&\ \ \boldsymbol{g} + (\beta+\gamma)\sum_{i=1}^{t-1}\boldsymbol{H}\boldsymbol{g}_i + \gamma\boldsymbol{H}\boldsymbol{g}_t\\
                & &=&\ \ \boldsymbol{g} + (\beta+\gamma)\sum_{i=1}^{t-1}\boldsymbol{H}(a_i\cdot\boldsymbol{g}+b_i\cdot\boldsymbol{Hg}) + \gamma \boldsymbol{H}(a_t\cdot\boldsymbol{g}+b_t\cdot\boldsymbol{Hg})\\
                & &\approx&\ \ \boldsymbol{g} + (\beta+\gamma)\sum_{i=1}^{t-1}a_i\cdot\boldsymbol{Hg} + \gamma\cdot a_t\boldsymbol{Hg}\\
                & &=&\ \ \boldsymbol{g} + \left((\beta+\gamma)\sum_{i=1}^{t-1}a_i+\gamma\cdot a_t\right)\boldsymbol{Hg},\\
            \end{aligned}
        \end{equation}
        where the first approximation holds as we ignore the high-order terms of $g^{(n)}\ (n\geq 2)$ and the second approximation holds as we ignore the term related to $\boldsymbol{HH}$. Then we have 
        \begin{equation}
            \begin{aligned}
                & a_{t+1} &=&\ \ 1,\\
                & b_{t+1} &=&\ \ (\beta+\gamma)\sum_{i=1}^{t-1}a_i+\gamma\cdot a_t = \beta\cdot(t-1)+\gamma\cdot t.    
            \end{aligned}
        \end{equation}
        Furthermore, 
        \begin{equation}
            \begin{aligned}
                & \boldsymbol{\delta}_{t+1} &=&\ \ \boldsymbol{\delta}_t + \boldsymbol{g}_{t+1}\\
                & &=&\ \ c_t\cdot\boldsymbol{g} + d_t\cdot\boldsymbol{Hg} + a_{t+1}\cdot\boldsymbol{g} + b_{t+1}\cdot\boldsymbol{Hg}\\
                & &=&\ \ (c_t+a_{t+1})\cdot\boldsymbol{g} + (d_t+b_{t+1})\cdot\boldsymbol{Hg}\\
                & &=&\ \ c_{t+1}\cdot\boldsymbol{g} + d_{t+1}\cdot\boldsymbol{Hg}
            \end{aligned}
        \end{equation}
        Then we have
        \begin{equation}
            \begin{aligned}
                & c_{t+1} - c_t &=&\ \ a_{t+1} = 1,\\
                & c_t - c_{t-1} &=&\ \ a_{t} = 1,\\
                & &\cdots&\ \ \\
                & c_3 - c_2 &=&\ \ a_3 = 1,\\
                & c_{t+1} &=&\ \ t+1.
            \end{aligned}
        \end{equation}
        Meanwhile, it holds
        \begin{equation}
            \begin{aligned}
                & d_{t+1} - d_t &=&\ \ b_{t+1},\\
                & d_t - d_{t-1} &=&\ \ b_t,\\
                & &\cdots&\ \ \\
                & d_3 - d_2 &=&\ \ b_3,
            \end{aligned}
        \end{equation}
        and
        \begin{equation}
            \begin{aligned}
                & d_{t+1} &=&\ \ \sum_{i=2}^{t+1} \beta(i-2) + \gamma(i-1)\\
                & &=&\ \ \frac{(t-1)(t-2)}{2}\beta+\frac{t(t-1)}{2}\gamma.
            \end{aligned}
        \end{equation}
    \end{enumerate}
\end{proof}

\goldbach*

\begin{proof}
    By \cite{DBLP:conf/iclr/WangRLZ0Z21}, the Shapley interactions between adversarial perturbation $i$ and $j$ is defined as
    \begin{equation}
        \boldsymbol{I}_{i,j}(\boldsymbol{\delta}_t) = \boldsymbol{\delta}_t(i)\cdot\boldsymbol{H}_{i,j}\cdot\boldsymbol{\delta}_t(j)+\mathcal{R}(\boldsymbol{\delta}_t(i),\boldsymbol{\delta}_t(j)),
    \end{equation}
    where $\boldsymbol{\delta}_t(i)$ represents the $i$-th element of $\boldsymbol{\delta}_t$, $\mathcal{R}(\cdot,\cdot)$ is the high order terms with respect to $\boldsymbol{\delta}_t(i)$ and $\boldsymbol{\delta}_t(j)$, and $\boldsymbol{H}_{i,j}$ is the element of the Hessian matrix in of $i$-th row and $j$-th column as
    \begin{equation}
        \boldsymbol{H}_{i,j} = \frac{\partial \boldsymbol{L}(\boldsymbol{x})}{\partial x_i \partial x_j},\ \boldsymbol{x} =(x_1,\ \dots,\ x_n)^\top.
    \end{equation}
    Then we have
    \begin{equation}
        \begin{aligned}
            & \boldsymbol{I}_{i,j}(\boldsymbol{\delta}_t) &\approx&\ \ (c_t\cdot \boldsymbol{g}(i) + d_t\cdot\boldsymbol{g}^\top\boldsymbol{H}_{*i})\boldsymbol{H}_{i,j}(c_t\cdot \boldsymbol{g}(j) + d_t\cdot\boldsymbol{g}^\top\boldsymbol{H}_{*j})\\
            & &=&\ \ c^2_t\cdot\boldsymbol{g}(i)\cdot\boldsymbol{g}(j)\cdot\boldsymbol{H}_{i,j} + c_t\cdot d_t\cdot\boldsymbol{g}(j)\cdot\boldsymbol{H}_{i,j}\cdot\boldsymbol{g}^\top\boldsymbol{H}_{*i}+c_t\cdot d_t\cdot\boldsymbol{g}(i)\cdot\boldsymbol{H}_{i,j}\cdot\boldsymbol{g}^\top\boldsymbol{H}_{*j}+\boldsymbol{\mathcal{O}}(\boldsymbol{H}^2)\\
            & &\approx&\ \ c^2_t\cdot\boldsymbol{g}(i)\cdot\boldsymbol{g}(j)\cdot\boldsymbol{H}_{i,j} + c_t\cdot d_t\cdot\boldsymbol{g}(j)\cdot\boldsymbol{H}_{i,j}\cdot\boldsymbol{g}^\top\boldsymbol{H}_{*i}+c_t\cdot d_t\cdot\boldsymbol{g}(i)\cdot\boldsymbol{H}_{i,j}\cdot\boldsymbol{g}^\top\boldsymbol{H}_{*j}.
        \end{aligned}
    \end{equation}
    Furthermore
    \begin{equation}
        \begin{aligned}
            & \boldsymbol{\mathbb{E}}_{i,j}[\boldsymbol{I}_{i,j}(\boldsymbol{\delta}_t)] &=&\ \  \boldsymbol{\mathbb{E}}_{i,j}[c^2_t\cdot\boldsymbol{g}(i)\cdot\boldsymbol{g}(j)\cdot\boldsymbol{H}_{i,j}+c_t\cdot d_t\cdot\boldsymbol{g}(j)\cdot\boldsymbol{H}_{i,j}\cdot\boldsymbol{g}^\top\boldsymbol{H}_{*i}+c_t\cdot d_t\cdot\boldsymbol{g}(i)\cdot\boldsymbol{H}_{i,j}\cdot\boldsymbol{g}^\top\boldsymbol{H}_{*j}]\\
            & &=&\ \ c_t^2\boldsymbol{\mathbb{E}}_{i,j}[\boldsymbol{g}(i)\cdot\boldsymbol{g}(j)\cdot\boldsymbol{H}_{i,j}]+2c_t\cdot d_t\cdot\boldsymbol{\mathbb{E}}_{i,j}[\boldsymbol{g}(i)\cdot\boldsymbol{H}_{i,j}\cdot\boldsymbol{g}^\top\boldsymbol{H}_{*j}]\\
            & &=&\ \ t^2 A + [t(t-1)(t-2)\beta+t^2(t-1)\gamma] B\\
            & &=&\ \ (\beta+\gamma)B\cdot t^3 + (A-3\beta B-\gamma B)\cdot t^2 + 2\beta B\cdot t
        \end{aligned}
    \end{equation}
    where
    \begin{equation}
        \begin{aligned}
            & A &=&\ \ \mathbb{E}_{i,j}[\boldsymbol{g}(i)\cdot\boldsymbol{g}(j)\cdot\boldsymbol{H}_{i,j}],\\
            & B &=&\ \ \boldsymbol{\mathbb{E}}_{i,j}[\boldsymbol{g}(i)\cdot\boldsymbol{H}_{i,j}\cdot\boldsymbol{g}^\top\boldsymbol{H}_{*j}].
        \end{aligned}
    \end{equation}
    Meanwhile, considering the updated rule of SAG, we have
    \begin{equation}
        \begin{aligned}
        & \boldsymbol{\mathbb{E}}_{i,j}[\boldsymbol{I}^{(t)}_{i,j}(\boldsymbol{\zeta}_t)]&=&\ \ t^2\boldsymbol{\mathbb{E}}_{i,j}[\boldsymbol{g}(i)\cdot\boldsymbol{g}(j)\cdot\boldsymbol{H}_{i,j}]+t^2(t-1)\cdot\boldsymbol{\mathbb{E}}_{i,j}[\boldsymbol{g}(i)\cdot\boldsymbol{H}_{i,j}\cdot\boldsymbol{g}^\top\boldsymbol{H}_{*j}]\\
        & &=&\ \ t^2\cdot A + t^2\cdot(t-1)\cdot B
        \end{aligned}
    \end{equation}
\end{proof}
\end{document}